\documentclass[twoside,11pt]{article}

\usepackage{jmlr2e}

\ShortHeadings{A Cluster Elastic Net for Multivariate Regression}{Price and Sherwood}

\usepackage{amsmath}
\usepackage{graphics}
\usepackage{setspace}
\usepackage{natbib}
\usepackage{fullpage}
\usepackage{url}
\usepackage{natbib}
\usepackage{algorithm}
\usepackage{color}

\newcommand{\vbeta}{\mbox{\boldmath $\beta$}}

\newcommand{\Real}{\mathcal{R}}
\newcommand{\va}{{\bf a}}
\newcommand{\vb}{{\bf b}}

\newcommand{\vW}{{\bf W}}

\newcommand{\vE}{{\bf E}}

\newcommand{\vu}{{\bf u}}
\newcommand{\vv}{{\bf v}}
\newcommand{\vw}{{\bf w}}
\newcommand{\vx}{{\bf x}}
\newcommand{\vy}{{\bf y}}
\newcommand{\vz}{{\bf z}}
\newcommand{\vA}{{\bf A}}
\newcommand{\vB}{{\bf B}}

\newcommand{\vH}{{\bf H}}

\newcommand{\vR}{{\bf R}}

\newcommand{\vU}{{\bf U}}

\newcommand{\vX}{{\bf X}}
\newcommand{\vY}{{\bf Y}}
\newcommand{\vZ}{{\bf Z}}

\newcommand{\Fs}{\mathcal{F}}

\newcommand{\vepsilon}{\mbox{\boldmath $\epsilon$}}

\newcommand{\vDelta}{\mbox{\boldmath $\Delta$}}

\newcommand{\vtheta}{\mbox{\boldmath $\theta$}}

\newcommand{\veta}{\mbox{\boldmath $\eta$}}

\newcommand{\Bin}{Bin}
\newcommand{\argmin}{\qopname\relax m{arg\,min}}

\newcommand{\bqa}{\begin{eqnarray*}}
\newcommand{\eqa}{\end{eqnarray*}}
\newcommand{\bqan}{\begin{eqnarray}}
\newcommand{\eqan}{\end{eqnarray}}

\firstpageno{1}

\begin{document}
\title{A Cluster Elastic Net for Multivariate Regression}
\author{\name Bradley S. Price \email brad.price@mail.wvu.edu \\
       \addr College of Business and Economics\\
       West Virginia University\\
       Morgantown, WV 26505, USA
       \AND
       \name Ben Sherwood \email ben.sherwood@ku.edu\\
       \addr School of Business\\ 
       University of Kansas\\
       Lawrence, KS 66045, USA}

\editor{ }

\maketitle

\begin{abstract}%

 
We propose a method for simultaneously estimating regression coefficients and clustering response variables in a multivariate regression model, to increase prediction accuracy and give insights into the relationship between response variables. The estimates of the regression coefficients and clusters are found by using a penalized likelihood estimator, which includes a cluster fusion penalty, to shrink the difference in fitted values from responses in the same cluster, and an $L_1$ penalty for simultaneous variable selection and estimation.  We propose a two-step algorithm, that iterates between k-means clustering and solving the penalized likelihood function assuming the clusters are known, which has desirable parallel computational properties obtained by using the cluster fusion penalty.  If the response variable clusters are known \emph{a priori} then the algorithm reduces to just solving the penalized likelihood problem.  Theoretical results are presented for the penalized least squares case, including asymptotic results allowing for $p \gg n$. We extend our method to the setting where the responses are binomial variables. We propose a coordinate descent algorithm for the normal likelihood and a proximal gradient descent algorithm for the binomial likelihood, which can easily be extended to other generalized linear model (GLM) settings. Simulations and data examples from business operations and genomics are presented to show the merits of both the least squares and binomial methods.  
 
\end{abstract}
 
  
\section{Introduction}
In this article we consider the pair 
$(\vx_i,\vy_i)_{i=1}^n$, with $\vx^T_i=(x_{i1},\ldots, x_{ip}) \in \Real^{p}$ and
$\vy_i=(y_{i1},\ldots, y_{ir})^T \in \Real^r$. Define $X = \left(\vx_1,\ldots,\vx_n\right)^T \in \Real^{n \times p}$ and $Y = \left(\vy_1,\ldots,\vy_n\right)^T \in \Real^{n \times r}$. We initially assume the linear model 

\begin{equation}
\label{lin_model}
\vy_i= {B^*}^T\vx_i+\vepsilon_i,
\end{equation}
where  $\vepsilon_i = (\epsilon_{i1},\ldots,\epsilon_{ir})^T \in \Real^r$ are realizations of an
 i.i.d. random variable with mean zero and covariance matrix $\Sigma$, $B^* = \left( \vbeta^*_1, \ldots, \vbeta^*_r\right) \in \Real^{p \times r}$ and $\vbeta_k^* = (\beta_{1k}^*,\ldots,\beta_{pk}^*)^T \in \Real^{p}$. We will refer to the matrix of error as $E = (\vepsilon_1,\ldots,\vepsilon_n)^T \in \Real^{n \times r}$.  Under mild assumptions a consistent estimator of $\vbeta_k^*$ is the ordinary least squares (OLS) estimator of 
\begin{equation*}
\tilde{\vbeta}_k =  \underset{\vbeta_k}{\mbox{argmin}} \sum_{i=1}^n (y_{ik} - \vx_i^T\vbeta_k)^2.
\end{equation*}
If  $\vepsilon_i$ are i.i.d. and $\vepsilon_i \sim N(\mathbf{0}_r, \Sigma)$ 
the estimator $\tilde{\vbeta}_k$ is the MLE. This estimator does not use the other responses, ignoring potentially useful information. 

Throughout this paper for a vector $\va$ define $||\va||_q$ as the $L_q$ norm and for a matrix $A$ we define $||A||_q$ as the entrywise $L_q$ norm. If there is \emph{a priori} information that the fitted values of response $k$ and $m$ should be close then we could impose a penalty on the difference in the fitted values and consider the estimators 
\begin{equation}
\label{group_ext}
(\tilde{\vbeta}_k,\tilde{\vbeta}_m) = \underset{\vbeta_k,\vbeta_m}{\mbox{argmin}} \sum_{i=1}^n \left\{ (y_{ik} - \vx_i^T\vbeta_k)^2 + (y_{im} - \vx_i^T\vbeta_m)^2\right\} + \frac{\gamma}{n} ||X(\vbeta_k-\vbeta_m)||_2^2,
\end{equation} 
where $\gamma$ is a tuning parameter controlling the amount of agreement between the two fitted values vectors. We propose an objective function that generalizes \eqref{group_ext} for multiple responses from multiple clusters that may not be known \emph{a priori}.  The proposed objective function also includes an $L_1$ penalty for simultaneous estimation and variable selection, which allows our method to be used to increase prediction accuracy, select relevant variables for each response, and detect groupings of response variables without assuming or estimating a covariance structure.  In our theory, simulations, and applied examples we consider cases where $p \gg n$. We extend the proposed method to the generalized linear model framework, specifically focusing on multiple binary responses. This extension allows the method to be used in many different contexts, such as understanding co-morbidities related to patient information recorded in electronic medical records, or product level purchasing habits of customers based on information obtained from a loyalty program. 
We propose a coordinate descent algorithm for the least squares case and proximal coordinate descent algorithm for the binomial GLM case, which provides a general framework for extending the method to other GLM or M-estimator settings. 
 
%
%
%
%
%

Our work has been influenced by previous work in estimating high dimensional models. When $\frac{1}{n}X'X = I_{p}$ the penalty function is equivalent to a ridge penalty \citep{ridge} on the difference of the coefficient vectors for the two responses. We add the $L_1$ penalty as proposed in \citet{lasso} to do simultaneous variable selection and estimation. Similar to the work of \citet{zou05} we combine the ridge and $L_1$ penalties. The proposed estimator simultaneously estimates clusters of the response and fuses the fitted values of the clustered responses. Previous work has been done on clustering covariates for high dimensional regression with a univariate response. This work is most similar to the work of \citet{witten14} who proposed the cluster elastic net (CEN) that simultaneously estimates clusters of covariates and fuses the effects of covariates within the same cluster. Our proposed method is also similar to Grace estimators proposed in \citet{li_08} and \citet{li_10}, which use regularization based on external network information to minimize the difference of coefficients for related predictors and use a lasso penalty for sparsity. \citet{huang_11} proposed the sparse Laplacian shrinkage method, which preforms variable selection and promotes similarities among coefficients of 
correlated covariates. \citet{zhao_16} proposed the Grace Test, a testing framework for Grace estimators, that allows for some uncertainty in the graph and showed that if the external graph is informative it increases  the power of the Grace test. \citet{buhlmanCluster} proposed two different penalized methods for clustered covariates in high-dimensional regression: cluster representative lasso (CRL) and cluster group lasso (CGL). In CRL the covariates are clustered, dimension reduction is done by replacing the original covariates with the cluster centers and a lasso model is fit using the cluster centers as covariates. In CGL the group penalty of \citet{yuan2007} is applied using the previously found clusters as the groups. \citet{bird} demonstrated that averaging over models using different cluster centers for both responses and predictors can improve prediction accuracy of DNase I hypersensitivity using gene expression data. \citet{kim_09} proposed graph-guided fused lasso (GGFL) to the specific problem of association analysis to quantitative trait networks. GGFL presents a fused lasso framework in multivariate regression that leverages correlated traits based on a network structure.  Our work is related to the fused lasso literature as well, though we do not achieve exact fusion \citep{tibs05, rinaldo2009, hoefling2010, tibs_14}. The proposed method differs from the works mentioned in this setting because it focuses on using correlation between the response variables to improve estimation, however all of the works mentioned were instrumental in helping us derive our final estimator. 
 

The idea of using information from different responses to improve estimation in multivariate regression is not new and our work builds upon previous works in this area. \citet{breiman97} introduced the Curds and Whey method whose predictions are an optimal linear combination of least squares predictions. \citet{rothman2010} proposed multivariate regression with covariance estimation (MRCE), which is a penalized likelihood approach to simultaneously estimate the regression coefficients and the inverse covariance matrix of the errors. MRCE leverages correlation in unexplained variation to improve estimation, while our proposed method leverages correlation in explained variation to improve estimation. Other estimators assume both the response and covariates are multivariate normal and exploit this structure to derive estimators \citep{lee2012, molstad}. \citet{rai_12} proposed a penalized likelihood method for multivariate regression that simultaneously
estimates regression coefficients, the inverse covariance matrix of the errors, and the covariance matrix of the regression coefficients across responses using lasso type penalties. 
\citet{peng10} introduced regularized multivariate regression for identifying master predictors (remMap), which relies on \emph{a priori} information about valuable predictors and imposes a group $L_1$ and $L_2$ norm, across responses, on all covariates not prespecified as being useful predictors.  \citet{kim_12} proposed the tree guided group lasso, which uses an \emph{a priori} hierarchical clustering of the responses 
to define overlapping group lasso penalties for the multivariate regression model.  They propose a weighting method that ensures all coefficients are penalized equally, while using the hierarchical structure to impose a similar sparsity structure across highly correlated responses. 

Another approach to improving efficiency is by doing dimension reduction on $Y$ to find a smaller subspace that retains the material information needed for estimation of the regression coefficients \citep{cook2010, cook2015, sprem}. \citet{cook2010} introduced the envelope estimator for the multivariate linear model, which projects the maximum likelihood estimator onto the estimated subspace with the material information. \cite{cook2015} provided envelope models for GLMs and weighted least squares. \cite{sprem} proposed a sparse regression model (SPReM) for estimating models where $r$ is very large. SPReM projects the response variables into a lower-dimensional space while maintaining the structure needed for a specific hypothesis test. The key difference between our proposed method and these approaches is that we are interested in simultaneously estimating clustering of the response variables and fusing the fitted values from responses within the same cluster.

The proposed method simultaneously estimates clusters of the response and coefficients. Changes in cluster groups are discrete changes and as a result our objective function is discontinuous, similar to k-means clustering, thus making it difficult to derive an efficient algorithm that will find the optimal estimates for coefficients and groups. \citet{witten14} dealt with a similar difficulty for the CEN estimator, but noticed that if the groups are fixed then the problem is convex, while if the regression coefficients are fixed the problem becomes a k-means clustering problem. We modify the approach proposed in \citet{witten14} to our problem of grouping responses and extend the approach to the case of generalized linear models, specifically the binomial logistic model.  
In our theoretical results we assume the clustering groups are known, but the problem remains challenging as we are dealing with multiple responses, allow for $p \gg n$ and for $p$ to increase with $n$. 

In Section 2 we present our method for the multivariate linear regression model and provide theoretical results, including consistency of our estimator, to better understand the basic properties of the penalized likelihood solution.  In Section 3 we provide details on the two-step iterative algorithm and show estimating the regression coefficients for the different clusters is an embarrassingly parallel problem, which is a property of our cluster fusion penalty that 
fuses within group fitted values. 
This avoids issues that would arise in fusing all possible combinations of regression coefficients, or having to specify a fusion set \emph{a priori}. Examples of the issues that can arise can be 
found in \cite{price17}, 
who discussed 
 the importance of choosing the fusion set, and the original fused lasso paper which fused only consecutive coefficients \citep{tibs05}.  In Section 4 we present the model for binomial responses along with an algorithm, demonstrating how the use of the cluster fusion penalty can exploit relationships of response variables beyond the traditional Gaussian problem. Simulations for both conditional Gaussian and binomial responses are presented in Section 5. The least squares version of our method is applied to model baby birth weight, placental weight and cotinine levels given maternal gene expression and demographic information. The binomial case is applied to model concession stand purchases using customer information as covariates. Both applied analysis are presented in Section 6. We conclude with a summary in Section 7. 
 

\section{Least Squares Model}
\subsection{Method}
First, we consider estimating \eqref{lin_model} when there are $Q$ unknown clusters of the $r$ responses. We further assume that $\sum_{i=1}^n y_{ik}=0$ for all $k=1,\ldots, r$, 
$\sum_{i=1}^n x_{ij}=0$ and $\sum_{i=1}^nx_{ij}^2 \leq n$ for all $j=1\ldots,p$. The model requires $rp$ parameters to be estimated for prediction, which is problematic when $r$ or $p$ are large. Let $D=(D_1,\ldots, D_Q)$ be a partition of the set $\{1,\ldots, r\}$. 
For a set $A$ define $|A|$ as the cardinality of that set. We propose the multivariate cluster elastic net (MCEN) estimator as
\begin{equation}
\label{opt}
\begin{split}
(\hat{B},\hat{D})= \argmin_{B \in \Real^{p \times r}, D_1,\ldots, D_Q} \frac{1}{2n}\sum_{i=1}^n\sum_{c=1}^r (y_{ic} -\vx_i^T\vbeta_c)^2 +\delta ||B||_1\\
+\frac{\gamma}{2n}\sum_{q=1}^Q\frac{1}{|D_q|}\sum_{l,m \in D_q} ||X(\vbeta_l-\vbeta_m)||_2^2,
\end{split}
\end{equation}
where $Q$ is the number of clusters and $\gamma$ and $\delta$ are non-negative user specified tuning parameters. In addition $Q$, the total number of clusters, can be considered a tuning parameter. The cluster fusion penalty, associated with tuning parameter $\gamma$, is used to exploit similarities in the fitted values. The lasso penalty, with tuning parameter $\delta$, is used to perform simultaneous estimation and variable selection. When $\gamma=0$ or $Q=r$, the optimization in \eqref{opt} reduces to $r$ independent lasso penalized least squares problems with tuning parameter $\delta$. If $\hat{D}$ is known then the optimization in \eqref{opt} can be split into $Q$ independent optimizations that are similar to the optimizations presented in \citet{li_08}, \citet{li_10}, and \citet{witten14} and can be solved in parallel. We exploit this computational feature in our algorithm, which is a result of using the cluster fusion penalty.   

The proposed method uses a combination of $L_1$ and $L_2$ penalties as proposed by \citet{zou05}. Similar methods have been proposed for grouping the effects of predictors with a univariate response such as CEN \citep{witten14} and Grace estimators \citep{li_08,li_10,zhao_16}. \citet{kim_12} proposed a method that uses a predetermined hierarchical clustering of the responses that provides an $L_1$ penalty for all coefficients and a group $L_2$ penalty for responses that are grouped together. \citet{chen_2016} proposed a method using conjoint clustering to 
incorporate similarities in preferences between individuals in conjoint analysis. 
This method does not simultaneously estimate coefficients and groupings.  It requires a two-step algorithm to estimate the number of clusters, and then estimates coefficients  
using  regularization based on  the estimated cluster. 
The proposed approach uses non-hierarchical clusters, allows for the clustering structure to be unknown before estimation of the coefficients and focuses more on imposing similar fitted values for grouped responses, compared to directly imposing a similar sparsity structure.   

Selecting the triplet, $(Q,\gamma,\delta)$, of tuning parameters can be done by K-fold cross validation minimizing the squared prediction error. 
Let $\Fs_k$ be the set of indices in the $k$th fold, $k \in\{1,\ldots, K\}$, and $\hat{\vbeta}^{(-\Fs_k)}_{c}(Q,\gamma,\delta)$ be the estimated regression coefficient vector using $Q$, $\gamma$ and $\delta$ for response $c$ produced from the training set with $\Fs_k$ removed. Then select the triplet, $(\hat{Q},\hat{\gamma},\hat{\delta})$, that minimizes 
\begin{equation}
V(Q,\delta,\gamma)= \sum_{k=1}^K\sum_{c=1}^r\sum_{i \in \Fs_k} \left\{y_{ic}-x_i^T\hat{\vbeta}^{(-\Fs_k)}_c(Q,\gamma,\delta)\right\}^2.
\end{equation}

\subsection{Theoretical Results}
\label{theory}

For theoretical discussions we assume that $D$ is known for some fixed value of $Q$. This is because for $D$ unknown the objective function in \eqref{opt} is discontinuous because of the discrete changes in groups, however if $D$ is known \eqref{opt} is a convex function. In this section we will look at properties of the MCEN estimator for the special case of fixed $n$ and $p$ with $\delta=0$. 
In addition, we present a consistency result that allows for $p \gg n$ when $\delta = o(1)$ and $\gamma=o(1)$. 

Thus, the first two theorems refer to the following estimator 

\begin{equation}
\label{opt_dfixed}
\begin{split}
\bar{B}= \argmin_{B \in \Real^{p \times r}} \frac{1}{2n}\sum_{i=1}^n\sum_{c=1}^r (y_{ic} -\vx_i^T\vbeta_c)^2 +\delta ||B||_1\\
+ \frac{\gamma}{2n}\sum_{q=1}^Q\frac{1}{|D_q|}\sum_{l,m \in D_q} ||X(\vbeta_l-\vbeta_m)||_2^2 .
\end{split}
\end{equation}

The estimator $\bar{B}$ does not simultaneously estimate the groups, it assumes they are known \emph{a priori}, and thus is different than $\hat{B}$. There are instances where the grouping structure is known before data analysis and thus using $\bar{B}$ would be preferable in practice. In addition $\bar{B}$ is a key component to the algorithm discussed in Section \ref{alg}. We begin by relating the estimator in \eqref{opt_dfixed} to ordinary least squares (OLS), for the special case of $\delta=0$. Removing the $L_1$ penalty allows us to derive a closed form for the estimator. 

\begin{theorem}
\label{thm1}
Assume $n>p$, $\delta=0$, and $Q$ and $\gamma$ are fixed values. Define $\dot{B}=(\dot{\vbeta}_1,\ldots,\dot{\vbeta}_r)$ to be the OLS estimates for the $r$ response variables and $\bar{B}=(\bar{\vbeta}_1,\ldots,\bar{\vbeta}_r)$ be the solution to \eqref{opt_dfixed} with tuning parameter $\gamma$.  Given $l \in D_q$ then 
$\bar{\vbeta}_l$ has the closed form solution of

\begin{equation}
\bar{\vbeta}_l=\dot{\vbeta}_l+\frac{2\gamma}{(1+2\gamma)|D_q|}\sum_{c \in D_q\\ c \neq l} (\dot{\vbeta}_c-\dot{\vbeta}_l).
\end{equation}
 
\end{theorem}

Theorem \ref{thm1} provides 
some intuition about the MCEN estimator.  As $\gamma$ increases the MCEN estimator approaches a weighted average of the OLS coefficients within a cluster. In addition the results from Theorem \ref{thm1} can be used to calculate the bias and variance of $\bar{B}$, which are needed for proving Theorem \ref{thm2}. The proof of Theorem \ref{thm1} and the following Theorems can be found in the appendix.




\begin{theorem}
\label{thm2}
Assume $E(\epsilon_{ic}^2)=1$ for all $i \in \{1,\ldots,n\}$ and $c \in \{1,\ldots,r\}$ and $E(\epsilon_{ic}\epsilon_{ik}) = \rho$ for $c \neq k$, where $\rho \in (0,1)$. Set $\delta=0$, then for a fixed $n$ and $p$ where $n>p$ there exists a positive $\gamma$ such that 
\begin{equation}
E\left(\left|\left|\bar{B}-B^*\right|\right|_2^2\right)\leq E\left(\left|\left|\dot{B}-B^*\right|\right|_2^2\right),
\end{equation}
where $B^*$ are the true regression coefficients, $\dot{B}$ is as defined in Theorem \ref{thm1} and $\bar{B}$ is as defined in \eqref{opt_dfixed}.  
\end{theorem}

Similar to ridge regression Theorem \ref{thm2} shows that for some positive $\gamma$ the estimator from \eqref{opt_dfixed} has a smaller mean squared error than OLS. Note, we are not assuming that for $l,s \in D_m$ that $\vbeta_l^*=\vbeta_m^*$ and unless this condition holds the estimator $\bar{B}$ is biased. Thus, there exists a value of $\gamma$ for which there is a favorable bias-variance trade off. 

Next we examine the asymptotic performance of the estimator with the $L_1$ penalty. At times it will be easier to refer to a vectorized version of a matrix and for any matrix $A \in \Real^{a \times b}$, $\mbox{vec}(A) \in \Real^{ab}$. Where $\mbox{vec}(A)$ is the vector formed by stacking the columns of $A$. 
Define $S$ as the set of active predictors. That is, S is a subset of $\{1,\ldots,rp\}$ where $m \in S$ if $\mbox{vec}(B^*)_m \neq 0$. The subspace for the active predictors is 
\begin{equation*}
\mathcal{M}(S) \equiv \{ \vtheta \in \Real^{pr} | \theta_j = 0 \mbox{ if } j \notin S \}.
\end{equation*} 
The parameter space will be separated using projections of vectors into orthogonal complements. We define a projection of a vector $\vu$ into space $\mathcal{M}(S)$ as  
\begin{equation*}
\vu_{\mathcal{M}(S)} \equiv \underset{\vv \in \mathcal{M}(S)}{\mbox{arg min}} ||\vu-\vv||_2.
\end{equation*}
The orthogonal complement of space $\mathcal{M}(S) \subseteq \Real^{p}$ is 
\begin{equation*}
\mathcal{M}^\perp(S) \equiv \{ \vv \in \Real^{pr} | \langle \vu, \vv \rangle =0 \mbox{ for all } \vu \in \mathcal{M}(S) \}.
\end{equation*}
The following set is central to our proof of consistency,
\begin{equation*}
\mathcal{C} \equiv \{ \vtheta \in \Real^{pr} \vert \, ||\vtheta_{\mathcal{M}^\perp(S)}||_1 \leq ||\vtheta_{\mathcal{M}}||_1 \}.
\end{equation*}

For our proof of the consistency of $\bar{B}$ we make the following six assumptions:
\begin{enumerate}
\item[A1] Define $\vX_j$ to be the $j$th column vector of $X$, then $\vX_j \in \Real^p$ has the condition that $\frac{||\vX_j||_2^2}{n}\leq 1$.
\item[A2] Define $\vepsilon_c = (\epsilon_{1c},\ldots,\epsilon_{nc})^T \in \Real^n$ as the error vector for response $c$. The error vector $\vepsilon_c$ has a mean of zero and sub-Gaussian tails for all $c \in \{1,\ldots,r\}$. That is, there exists a constant $\sigma_c$ such that for any $\va \in \Real^n$, with $||\va||_2 = 1$, 
\begin{equation*}
P(|\langle \vepsilon_c, \va \rangle| > t) \leq 2 \mbox{exp} \left( - \frac{t^2}{2\sigma_c^2} \right).
\end{equation*}
Define $\sigma = \underset{c}{\max} \, \sigma_c$. 
\item[A3] 
Define $\tilde{X} = I_r \otimes X \in \Real^{rn \times rp}$, where $\otimes$ is the standard Kronecker product. There exists a positive constant $\kappa$ such that 
\begin{equation*}
\kappa ||\vtheta||_2^2 \leq \underset{\vtheta \in \mathcal{C}}{\mbox{min}} \, n^{-1}||\tilde{X}\vtheta||_2^2.
\end{equation*}
\item[A4] There exists a positive constant $\acute{b}$ such that $\max_{q=1,\ldots,Q}\max_{(l,k) \in D_q} ||\vbeta_l^*-\vbeta_k||_2 \leq \acute{b}$.
\item[A5]  Given $l,k \in D_q$, if $\beta^*_{lj}=0$ then $\beta^*_{kj}=0$, for all $j \in \{1,\ldots,p\}$ and $q \in \{1,\ldots, Q\}$.  
\item[A6] Define $\rho_{\max}(A)$ as the maximum eigenvalue of square matrix $A$ and $X_{S_{D_q}}$ as the matrix of true predictors for cluster $q$, where the $j$th predictor is a true predictor if $\vbeta^*_{lj} \neq 0$ for any $l \in D_q$. There exists a positive constant $\rho_{\max}$ such that  
$$
\max_{q=1,\ldots,q} \rho_{\max}\left(\frac{1}{n}X_{S_{D_q}}^TX_{S_{D_q}}\right)\leq \rho_{\max}.
$$
\end{enumerate}
Assumption A1 is a standard assumption for lasso-type penalties and can be achieved by appropriately scaling the covariates, which is commonly done in penalized regression. Assumption A2 is a generalization of the sub-Gaussian error assumption for penalized regression for a univariate response. Assumption A1 could be relaxed to allow for certain unbounded covariates, but then A2 would be replaced by assuming the errors are normally distributed \citep{dantzig, lassoTypeRecovery}. Assumption A3 is a generalization of the common restricted eigenvalue assumption. Motivation for assumption A3 is discussed in great detail by \citet{negahban2012} and a version for $r=1$ has been used in several works analyzing asymptotic behaviors of the lasso estimator \citep{bickel2009, oracleLasso, lassoTypeRecovery}. 
Assumptions A4 and A5 provide that the true coefficients are similar for responses in the same group. Assumption A5 provides that they have the same sparsity structure. While, assumption A4 ensures that the difference in the non-zero elements can be bounded by a finite constant, even if the number of predictors increases with $n$. Assumption A6 assumes the maximum eigenvalues of the sample covariance of the true predictors are bounded, a common assumption in high-dimensional work. 
Assumptions A4-A6 can be replaced by an assumption similar to assumption A2 from \citet{witten14} that if $b,c \in D_m$ then $\vbeta^*_b=\vbeta^*_c$ , for all $m \in \{1,\ldots, Q\}$, thus the bias of the MCEN estimator only comes from the $L_1$ penalty.



Using assumptions A3 and A5 we can provide a closed form definition of the asymptotic bias when $\delta=0$. This relationship will be central to our proof of consistency of $\bar{B}$.

\begin{corollary}
\label{cor1}
Let $B^*$ be an s-sparse matrix, whose column vectors are all sparse and $E[X^TX/n] \in \Real^{p \times p}$ to be a positive definite matrix. Assume $Q$ and $\gamma$  are fixed values.  Define,
$$
\acute{B}=\left(\acute{\beta}_1,\ldots\acute{\beta}_r\right)=\argmin_{\beta_1,\ldots,\beta_r \in \Real^p} E\left(\frac{1}{2n}\sum_{i=1}^n\sum_{c=1}^r (y_{ic} -\vx_i^T\vbeta_c)^2 \\
+ \frac{\gamma}{2n}\sum_{q=1}^Q\frac{1}{|D_q|}\sum_{l,m \in D_q} ||X(\vbeta_l-\vbeta_m)||_2^2\right),
$$
Assume $l \in D_q$ then $\acute{\beta}_l$ has closed form solution,
$$
\acute{\beta}_l=\beta^*_l+\frac{2\gamma}{(1+2\gamma)|D_q|)}\sum_{c \in D_q, c \neq l} (\beta^*_c-\beta^*_l).
$$
\end{corollary}

Corollary \ref{cor1} provides insight into what $\bar{B}$ would converge to for a fixed $\gamma$. Knowing this exact relationship is used in our consistency proof because it allows us to understand the exact nature of the bias caused by the $L_2$ penalty and for $\gamma$ going to zero at a given rate we can show that the bias is asymptotically negligible.

\begin{theorem}
\label{thm_strong}
Let $B^*$ be an s-sparse matrix, whose column vectors are all sparse and $E[X^TX/n] \in \Real^{p \times p}$ to be a positive definite matrix. 
Given 
$\delta =  16\sigma\sqrt{\frac{\log(rp)}{n}}$, $\gamma \leq \frac{5}{4\rho_{\max}\acute{b}}\sigma \sqrt{\frac{\log(rp)}{n}}$ and assumptions A1-A6 hold then there exist constants $c_1$, $c_2$, $c_3$ and $c_4$ such that
\begin{equation}
\left|\left|\mbox{vec}\left({\bar{B}}-{B^*}\right)\right|\right|_2 \leq \sigma\sqrt{\frac{s\log(rp)}{n}}\left(\frac{c_3}{\kappa}+\frac{c_4}{\rho_{\max}}\right),
\end{equation}
with probability at least $1-c_1\exp(-c_2n\delta^2)$.
\end{theorem}
The convergence rate derived is similar to rates found in lasso-type estimators with a univariate response, with $\log(rp)$ replacing $\log(p)$ to accommodate for the multiple responses \citep{bickel2009, dantzig, lassoTypeRecovery, negahban2012}. Thus, under the conditions of Theorem \ref{thm_strong} if $pr \rightarrow \infty$ then $||\mbox{vec}(\bar{B}-B^*)||_2 = O_p\left\{\sqrt{\frac{s \log(rp)}{n}}\right\}$. Our results 
prove consistency of our estimator when the group structure is known. \citet{zhao_16} propose the Grace test for an estimator with a similar penalty for grouping predictors with a univariate response and establish asymptotic results that allow for inference even if there is some uncertainty to the grouping structure.

\section{Algorithm}
\label{alg}

The optimization in \eqref{opt} is discontinuous because of the estimation of cluster assignments. 
To simplify the optimization we propose an iterative algorithm that alternates between 
estimating the groups with the regression coefficients fixed, and estimating the regression coefficients with the groups fixed.  
If the clusters are known \eqref{opt_dfixed} then it is a convex optimization problem that can be solved by a coordinate descent algorithm. Let $R=\frac{1}{n}X^TX$, define $\vR_j$ as the $j$th column of $R$. The super script $(-h)$ denotes the $h$th element of the vector has been removed, and $r_{jj}$ is $j$th diagonal element of $R$. Define $S(a,b) = \mbox{sign}(a)\max(0,|a|-b)$.  To solve \eqref{opt_dfixed}, we use a coordinate descent algorithm where each update is preformed by    
\begin{equation}
\label{MV_coordinate_descent}
\bar{\beta}_{jk} \leftarrow \frac{S\left[\frac{1}{n}y_{k}^T\vX_j-\left\{1+\frac{\gamma(\vert D_q \vert-1)}{\vert D_q \vert}\right\}\vR_j^{(-j)T}\overline{\vbeta}_k^{(-j)}+\frac{\gamma}{\vert D_q \vert}\sum_{s \in D_q, s \neq k}\vR_j^T\overline{\vbeta}_s, \delta/2\right]}{r_{jj}\left(1+\gamma\frac{|D_q|-1}{|D_q|}\right)}.
\end{equation}
Thus, \eqref{opt_dfixed} is solved by iterating through $j \in \{1,\ldots,p\}$ and $k \in \{1,\ldots,r\}$ until the solution converges, similar to other coordinate descent solutions \citep{witten14, li_10, li_08,friedman2008}. If $B$ is known then the solution to \eqref{opt} reduces to the well studied k-means clustering problem. Recognizing this, we propose a two-step iterative procedure to obtain a local minimum. To start the algorithm an initial estimate of $D$ or $B$ is needed. We propose initializing the regression coefficients for the different responses separately with the elastic net estimator of response $c$ of
\begin{equation}
\label{sen}
\hat{\vbeta}^1_c = \argmin_{\vbeta_c \in \Real^{p}} \frac{1}{2n}\sum_{i=1}^n (y_{ic} -\vx_i^T\vbeta_c)^2 +\delta ||\vbeta_c||_1 + \gamma ||\vbeta_c||_2^2,
\end{equation}   
where $\hat{B}^w = \left(\hat{\vbeta}_1^w,\ldots,\hat{\vbeta}_r^w\right)$ represents the $w$th iterative estimate of $B^*$. Given a fixed $(Q,\gamma,\delta)$ we propose the following algorithm.
\begin{enumerate}
\item Begin with initial estimates, $\hat{\vbeta}_{1}^1,\ldots, \hat{\vbeta}_{r}^1$. 
\item For the $w$th step, where $w > 1$, repeat the steps below until the group estimates do not change:
\begin{enumerate} 
\item Hold $\hat{B}^{w-1}$ fixed and minimize,
\label{step_2b_MVCEN}
\begin{equation}
\left( \hat{D}_1^w,\ldots, \hat{D}_Q^w\right)=\underset{D_1,...,D_Q}{\mbox{minimize}} \left\{\sum_{q=1}^Q \frac{1}{|D_q|} \sum_{l,m \in D_q} \left|\left|X\left(\hat{\vbeta}^{w-1}_l-\hat{\vbeta}^{w-1}_m\right)\right|\right|_2^2 \right\}.
\end{equation}
The above can be solved by performing $K$-means clustering on the $r$ $n-$dimensional vectors $X\hat{\vbeta}^{w-1}_1,\ldots,X\hat{\vbeta}^{w-1}_r$. 
\item   Holding $\hat{D}_1^w,\ldots,\hat{D}_Q^w$ fixed the $w$th estimate of $B^*$ is 
\begin{equation}
\label{step2b_opt}
\begin{split}
\hat{B}^w = \argmin_{B \in \Real^{p \times r}} \frac{1}{2n}\sum_{i=1}^n\sum_{c=1}^r (y_{ic} -\vx_i^T\vbeta_c)^2 +\delta ||B||_1\\
+ \frac{\gamma}{2n}\sum_{q=1}^Q\frac{1}{|\hat{D}_q^w|}\sum_{l,m \in \hat{D}_q^w} ||X(\vbeta_l-\vbeta_m)||_2^2 .
\end{split}
\end{equation}
Note that for the groups known, instead of estimated, $\hat{B}^w$ is equivalent to $\bar{B}$. Thus \eqref{step2b_opt} can be solved using the coordinate descent solution from \eqref{MV_coordinate_descent} using $\hat{B}^{w-1}$ as the initial estimates for the coordinate descent algorithm. 
\label{step_2c_MVCEN}
\end{enumerate}
\end{enumerate}

Convergence is reached once the groups at the $w$th and $(w-1)$th iteration are the same. The optimization in \eqref{opt_dfixed} is separable with respect to $\hat{D}$, and results in $Q$ independent optimization problems that can be solved in parallel. 
The algorithm for \eqref{opt_dfixed} can be solved in solution path type form where we iterate across different values of $\delta$ in a similar fashion as proposed in the glmnet algorithm \citep{friedman2008}. If all of the initial elastic net estimators are fully sparse, we set the solution to be a zero matrix and thus following 
\cite{friedman2008}, initialize the algorithm by beginning the sequence with $\delta_{\max}$ at
\begin{equation*}
\delta_{\max}=2 \max_{j,k} \left|\frac{\sum_{i=1}^n y_{ik}x_{ij}}{n}\right|.
\end{equation*}
Our two-step approach is closely related to the CEN algorithm proposed by 
\citet{witten14}, 
who proposed a two-step algorithm where the two steps are solved by coordinate descent and k-means algorithms. The major difference in our proposal is that we cluster the responses rather than the predictors, and have the ability to solve the optimization in parallel due to the nature of our regularization in a multiple response setting.
\section{Binomial Model}

\subsection{Method}
Next we extend the multivariate cluster elastic net to generalized linear models. We focus specifically on the binomial response case, 
but our discussion here will scale to other exponential families. A fusion penalty has been proposed for merging groups from a multinomial response  \citep{price17}, but our method differs as it aims to leverage association between multiple binomial responses.  \citet{kasap_16} proposed an ensemble method that combines association rule mining and binomial logistic regression via a multiple linear regression model.  Our method differs from this by simultaneously estimating the clusters of the response variables and estimating the regression coefficients.   
An example is $n$ customers, with $p$ covariates, such as demographic and historic purchasing variables, and $r$ indicators of product purchasing statuses for each customer. You could run $r$ independent models, but this would not allow for modeling the relationship between the different products. 
Extending 
the multivariate cluster elastic net 
to multiple binomial responses would allow us to group products by purchase probabilities to identify and use relationships between products.  This could also be used 
to create a probabilistic model for diseases based on patient demographic and medical information. 

For the linear model we ignore the intercept term as it can be removed by appropriately scaling $Y$ and $X$. This is not possible in logistic regression, therefore the model needs an intercept term. We define $\vu_i = (1,\vx_i^T)^T \in \Real^{p+1}$, $U = (\vu_1^T,\ldots,\vu_n)^T \in \Real^{n \times p+1}$, $\vU_k \in \Real^{n}$ as the kth column vector of $U$ and $\tilde{R} = U^TU$. The true coefficients for response $k$ is defined as $\vtheta_k^* \in \Real^{p+1}$, $\Theta^* = (\vtheta^*_1,\ldots,\vtheta^*_r) \in \Real^{p+1 \times r}$, $\Theta^*_{-1} \in \Real^{p \times r}$ is the matrix with the first row, the row of intercept coefficients, of $\Theta^*$ removed and $\vtheta^*_{(-1)k} \in \Real^{p}$ is the $k$th column vector of $\Theta^*_{-1}$. In this model $y_{ik}$ is an independent draw from 
\begin{equation}
\Bin\left(1,\pi^*_{ik}\right),
\end{equation}
where 
\begin{equation}
\label{logit}
\pi^*_{ik}=\frac{\exp(\vu_i^T\vtheta_k^*)}{1+\exp(\vu_i^T\vtheta_k^*)}.
\end{equation}
The penalized negative log-likelihood function is
\begin{equation}
\label{BinPen}
\begin{split}
& \sum_{k=1}^r \sum_{i=1}^n y_{ik}\vu_i^T\vtheta_k - \log\left\{1+\exp(\vu_i^T\vtheta_k)\right\} \\
& + \frac{\gamma}{2n}\sum_{q=1}^q \frac{1}{|D_q |} \sum_{l,m \in D_q} ||U(\vtheta_l-\vtheta_m)||_2^2 +\delta ||\Theta_{-1}||_1.
\end{split}
\end{equation}

\subsection{Algorithm}

We propose solving \eqref{BinPen} by approximating it with a penalized quadratic function similar to the glmnet algorithm proposed by \citet{friedman2008}.  Define,
\begin{equation}
g(\pi_{ik})=\log\left(\frac{\pi_{ik}}{1-\pi_{ik}}\right)=\vu_i^T\vtheta_k.
\end{equation}

To implement this approximation we define 
\begin{eqnarray}
z_{ik}&=&g(y_{ik})=g(\pi_{ik})+\frac{y_{ik}-\pi_{ik}}{\pi_{ik}(1-\pi_{ik})},\\
w_{ik}&=&\pi_{ik}(1-\pi_{ik}),\\
\label{quad_approx}
-l_{Ak}(\vtheta_k)&=&\sum_{i=1}^nw_{ik}(z_{ik}-\vu_i^T\vtheta_k)^2.
\end{eqnarray}
Note that $z_{ik}$ is just the first order Taylor approximation of $g(y_{ik})$, and that $w_{ik}$ is the conditional 
variance of $z_{ik}$ given $\vu_i$. Define $\vZ_k = (z_{1k},\ldots,z_{nk})^T \in \Real^n$ and $\vW = (w_{1k},\ldots,w_{nk})^T \in \Real^n$. 

The MCEN estimator for the binomial model is
\begin{equation}
\label{BinApprox}
\begin{split}
(\hat{\Theta},\hat{D})= \argmin_{\Theta \in \Real^{p+1 \times r}, D_1,\ldots, D_Q} &\sum_{k=1}^r -l_{Ak}(\vtheta_k)+ \delta || \Theta_{-1} ||_1\\
& +  \frac{\gamma}{2n}\sum_{q=1}^q \frac{1}{|D_q |} \sum_{l,m \in D_q} ||U(\vtheta_r-\vtheta_s)||_2^2.
\end{split}
\end{equation}

If the groups are known \emph{a priori} the solution is 
\begin{equation}
\label{BinApproxFixedGroups}
\begin{split}
\bar{\Theta}= \argmin_{\Theta \in \Real^{p+1 \times r}} &\sum_{k=1}^r -l_{Ak}(\vtheta_k)+ \delta || \Theta_{(-1)} ||_1\\
& +  \frac{\gamma}{2n}\sum_{q=1}^q \frac{1}{|D_q |} \sum_{l,m \in D_q} ||U(\vtheta_r-\vtheta_s)||_2^2.
\end{split}
\end{equation}
For same length vectors $\va$ and $\vb$ let $\va \circ \vb$ represent the component wise multiplication of the two vectors. To solve \eqref{BinApproxFixedGroups}, we use a proximal coordinate descent algorithm where each update is performed by    
\begin{equation}
\label{MVB_coordinate_descent}
\bar{\theta}_{jk} \leftarrow \frac{S\left\{(\vw_{k}\circ \vz_{k})^{T}\vU_j-M_{jk}, I(j\neq 0)\delta/2\right\}}{r_{jj}\gamma\frac{|D_q|-1}{n|D_q|}+\vU_j^T(\vw_k\circ \vU_j)},
\end{equation}
where
\begin{eqnarray*}
M_{jk}&=&\sum_{c=1, c\neq h}^p\vU_j^T(\vw_k\circ \vU_c)\bar{\Theta}_{cj}+\frac{\gamma(\vert D_q \vert-1)}{n\vert D_q \vert}\tilde{\vR}_j^{(-j)T}\bar{\vtheta}_k^{(-j)}\\
&-&\frac{\gamma}{n\vert D_q \vert}\sum_{s \in D_q, s \neq k}\tilde{\vR}_j^T\bar{\vtheta}_s.
\end{eqnarray*}
The final solution is found by iterating through $j \in \{1,\ldots,p\}$ and $k \in \{1,\ldots,r\}$ until convergence. Again this is a solution similar to the glmnet algorithm proposed by \citet{friedman2008}.

To solve \eqref{BinApprox}, we propose an algorithm that is similar in nature to the penalized
least squares solution proposed in Section \ref{alg}.  The main difference is that we solve \eqref{BinApprox} with $D_1,\ldots, D_Q$ fixed using an iteratively reweighed least squares (IRWLS) solution with a proximal coordinate descent algorithm. The initial estimator for each response is done separately with 
\begin{equation}
\label{glm_init}
\hat{\vtheta}^1_k = \argmin_{\vtheta_k \in \Real^{p+1}}  -l_{Ak}(\vtheta_k)+ \delta || \vtheta_{(-1)k} ||_1 + \gamma ||\vtheta_{(-1)k}||_2^2.
\end{equation}
The following is our proposed algorithm for estimating \eqref{BinApprox}. 
\begin{enumerate}
\item Begin with initial estimates of $\hat{\Theta}^1 = \left(\hat{\vtheta}^1_1,\ldots,\hat{\vtheta}^1_r\right) \in \Real^{p+1\times r}$.
\item For the $w$th step, where $w > 1$, repeat the steps below until the group estimates do not change:
\begin{enumerate} 
\item Hold $\hat{\Theta}^{w-1}$ fixed and minimize 
\label{step_2b_MVBCEN}
\begin{equation}
\begin{split}
\left(\hat{D}^w_1,\ldots, \hat{D}^w_Q\right)=\underset{D_1,...,D_Q}{\mbox{minimize}} \Biggl\{\sum_{q=1}^Q \frac{1}{|D_q|} \sum_{l,m \in D_q} & \Biggl|\Biggl|U\left(\vtheta^{w-1}_l-\vtheta^{w-1}_m\right) \Biggr|\Biggr|_2^2 \Biggr\}.
\end{split}
\end{equation}
The above can be solved by performing $K$-means clustering.

\item[(2b)]  Holding $\hat{D}_1^w,\ldots,\hat{D}_Q^w$ fixed the $w$th update for the coefficients is 
\begin{equation}
\label{glm_step2b}
\begin{split}
 \hat{\Theta}^w = \argmin_{\Theta \in \Real^{p+1 \times r}} &\sum_{k=1}^r -l_{Ak}(\vtheta_k)+ \delta ||  \Theta_{-1}||_1\\
& +  \frac{\gamma}{2n}\sum_{q=1}^q \frac{1}{|\hat{D}^w_q|} \sum_{l,m \in \hat{D}^w_q} ||U(\vtheta_r-\vtheta_s)||_2^2.
\end{split}
\end{equation}
Where \eqref{glm_step2b} can be solved using the proximal coordinate descent solution presented in \eqref{MVB_coordinate_descent}, using $\hat{\Theta}^{w-1}$ as the initial estimates for the proximal coordinate descent algorithm. 
\label{step_2c_MVBCEN}
\end{enumerate}
\end{enumerate}

The triplet $(Q,\gamma,\delta)$ can be selected using K-Fold cross validation maximizing the validation log-likelihood.  Let $\Fs_k$ be the set of indices in the $k$th fold $(k \in\{1,\ldots, K\})$ and $\hat{\pi}^{(-\Fs_k)}_{ic}(Q,\gamma,\delta)$ be the estimated probability for observation $i$ and response $c$ produced from the model with $\Fs_k$ removed using $Q$, $\gamma$ and $\delta$.  Specifically we select the triplet that maximizes
\begin{equation}
\label{glm_cv}
V(Q,\delta,\gamma)= \sum_{v=1}^K\sum_{c=1}^r\sum_{i \in \Fs_v}
\left[ y_{ic}\log\left\{\hat{\pi}^{(-\Fs_k)}_{ic}\right\}+(1-y_{ik})\log\left\{1-\hat{\pi}^{(-\Fs_k)}_{ic}\right\} \right].
\end{equation}
The quadratic approximation defined by
\eqref{quad_approx}, is a standard technique used to estimate
parameters in generalized linear models, making this framework and our algorithm 
scalable to other exponential family settings \citep{farawaybook}. Tuning parameter selection would then be done by 
updating \eqref{glm_cv} with the appropriate likelihood.    

\section{Simulations}

\subsection{Gaussian Simulations}
In this section we compare the performances of the MCEN estimator \eqref{opt}, the true MCEN (TMCEN) \eqref{opt_dfixed}, with clustering structure known \emph{a priori}, the separate elastic net (SEN) estimator \eqref{sen}, the joint elastic net (JEN) estimator  
\begin{equation}
\label{jen}
\hat{B}_{\mbox{JEN}} = \argmin_{B}
\frac{1}{2n}\sum_{k=1}^r\sum_{i=1}^n (y_{ik} -\vx_i^T\vbeta_k)^2 +\delta \sum_{j=1}^p \sqrt{\beta_{j1}^2+\ldots+\beta_{jr}^2} + \gamma \sum_{k=1}^r\sum_{j=1}^p \beta_{jk}^2,
\end{equation}
 and the tree-guided group lasso (TGL) \citep{kim_12}. Define $\vB_j \in \Real^{r}$ as the $j$th row vector of matrix $B$. Given a tree $T$ with vertices $V$, where each node $v \in V$ is associated with group $G_v$ define $B_{j}^{G_v}$ as a vector of the $j$th predictors from responses in group $G_v$. The TGL estimator is 
\begin{equation}
\label{tgl}
\hat{B}_{\mbox{TGL}} = \argmin_{B}
\frac{1}{2} \sum_{k=1}^r\sum_{i=1}^n (y_{ik} -\vx_i^T\vbeta_k)^2 +\delta \sum_{j=1}^p \sum_{v \in V} w_v ||B_j^{G_v}||_2,
\end{equation}
where $w_v$ are weights that can vary with the nodes. See \citet{kim_12} for a detailed presentation of TGL, including how the weights, $w_v$, are derived. 


The JEN and SEN models are fit using the \texttt{glmnet} package in R \citep{friedman2008}. Tuning parameters for all methods are selected using 10-folds cross validation. For the MCEN and TMCEN methods cluster sizes of 2, 3 and 4 are considered. We include the TMCEN estimator for two reasons. First, in practice the TMCEN estimator could be used if the practitioner has a predetermined clustering of the responses. Second, the TMCEN is useful as a benchmark to compare with the MCEN estimator because if the grouping of responses is useful TMCEN provides the optimal grouping.
In all of the simulations the sample size is 100 and the number of responses is 15. For the number of covariates we considered 12, 100 and 300. Next we define how the covariates are generated and then will present the generating process for the response variables. 

Define $\tilde{\Sigma} \in \Real^{12 \times 12}$ with entries $\tilde{\sigma}_{ii} = 1$ and $\tilde{\sigma}_{ij} = \rho$, for $i\neq j$. Let $0_{a,b} \in \Real^{a \times b}$ be a matrix with all entries equal to zero. The covariates are generated by $\vx_i \sim N(\mathbf{0}_p,\Sigma_x)$, where $\Sigma_x = \tilde{\Sigma}$ for $p=12$ and otherwise
\begin{equation*}
\Sigma_x = 
\left(
 \begin{array}{cc}
   \tilde{\Sigma} &  0_{p-12,p-12}\\
    0_{p-12,p-12} & I_{p-12}
 \end{array}
\right),
\end{equation*}
with $\rho =.7$. 

For a group of responses we define the grouped coefficients as $\vb_q(\eta,\lambda) = (\veta_q-\lambda,\veta_q^*,\veta_q+\lambda,\veta_q+2\lambda,\veta+3\lambda) \in \Real^{q \times 5}$, where $\lambda$ is a constant and $\veta_q \in \Real^q$ with each element equal to $\eta$.  In the case of $p=12$ the matrix of coefficients is 
\begin{equation*}
B_{\eta,\lambda}^*=\left\{\begin{array}{ccc}
\vb_4(\eta,\lambda) & \mathbf{0}_{4,5} & \mathbf{0}_{4,5} \\ 
\mathbf{0}_{4,5} & \vb_4(\eta,\lambda) & \mathbf{0}_{4,5} \\ 
\mathbf{0}_{4,5} & \mathbf{0}_{4,5} & \vb_4(\eta,\lambda)
\end{array} \right\},
\end{equation*}
otherwise
\begin{equation*}
B_{\eta,\lambda}^*=\left\{\begin{array}{ccc}
\vb_{10}(\eta,\lambda) & \mathbf{0}_{10,5} & \mathbf{0}_{10,5} \\ 
\mathbf{0}_{10,5} & \vb_{10}(\eta,\lambda) & \mathbf{0}_{10,5} \\ 
\mathbf{0}_{10,5} & \mathbf{0}_{10,5} & \vb_{10}(\eta,\lambda)\\
\mathbf{0}_{p-30,5} & \mathbf{0}_{p-30,5} & \mathbf{0}_{p-30,5}
\end{array} \right\}.
\end{equation*}

Define $\Sigma_\epsilon \in \Real^{15 \times 15}$ with $\sigma(\epsilon)_{ij}$ being the entry for the $i$th row and $j$th column of $\Sigma_\epsilon$. The generating process for the response is 
\begin{equation}
\vy_i  = {B^*_{\eta,\lambda}}^T\vx_i + \vepsilon_i,
\end{equation}
where $\vepsilon_i \sim N(\mathbf{0}_{15},\Sigma_\epsilon)$, $\sigma(\epsilon)_{ii} = 1$ and $\sigma(\epsilon)_{ij} = 0$, for $i$ not equal to $j$. In all simulations we set the sample size to 100, perform 50 replications and with $p$ set consider the following 9 different combinations for the true coefficient matrix, 
\begin{equation*}
(\eta,\lambda) \in \{0.25,0.5,0.75,1\} \times \{0.02, 0.05, 0.10\}.
\end{equation*}

Models are fit using the training data with a sample size of 100. The tree for TGL is defined by performing complete-linkage hierarchical clustering on the responses in the training data.  In addition we generate 1000 additional testing samples to assess the prediction accuracies of the models. Let $y^*_{ij}$ represent the $i$th training sample for the $j$th response and $\hat{y}_{ij}$ represent a predicted value of that sample and response. The average squared prediction error (ASPE) is defined as 
\begin{equation}
\frac{1}{15000} \sum_{i=1}^{1000} \sum_{j=1}^{15} \left( y_{ij}^* - \hat{y}_{ij}\right)^2.
\end{equation}
We also report the mean squared error (MSE) of the estimators where for an estimator $B$
\begin{equation}
\mbox{MSE}\left(B\right) = \sum_{j=1}^{15} \left|\left|\vbeta_j - \vbeta_j^*\right|\right|_2^2.
\end{equation}
In addition we report the number of true variables selected (TV), out of a maximum of 60 for $p=12$ and 150
 otherwise, and the number of false variables selected (FV). Box plots of the statistics for $p=300$ and the different combinations of $\eta$ and $\lambda$ are reported in Figures \ref{fig:mse_p_300}--\ref{fig:fv_p_300}. These results show that TMCEN generally outperforms all methods in terms of ASPE and MSE. The one exception being when $\eta=1$, particularly for larger values of $\lambda$, TGL is competitive with or outperforms TMCEN. For larger values of $\lambda$ we expect more bias in the MCEN and TMCEN solutions and our simulation setting is favorable to TGL because the sparsity structure is the same for responses in the same cluster. With regards to ASPE and MSE, MCEN generally does better than TGL when $\eta = .5 \mbox{ or } .75$. This suggests that the MCEN approach is advantageous with several smaller signals, but the signals need to be strong enough to correctly identify the clustering of the responses. The MCEN method also outperforms JEN and SEN in terms of ASPE and MSE, except in the case of $\eta=.25$ where JEN outperforms MCEN. In this case the signal is too small resulting in the MCEN method not finding the true clustering structure, and thus the grouping penalty will not be optimal. The MCEN and TMCEN methods tend to pick a larger model than SEN, but a smaller model than JEN. This results in the MCEN and TMCEN methods correctly choosing more true predictors than SEN and fewer false positive predictors than JEN for weaker signal cases. In terms of variable selection MCEN and TMCEN tend to do better than TGL in terms of both true and false variable selection. For the stronger signal cases the SEN approach does the best in terms of variable selection, tending to have the maximum number of true covariates selected, while a smaller number of false covariates selected. Similar conclusions can be derived for the plots of $p=12$ and $p=100$, which are available in the supplementary material. 


%
%
\begin{figure}
	\centering
		\includegraphics[width=4in]{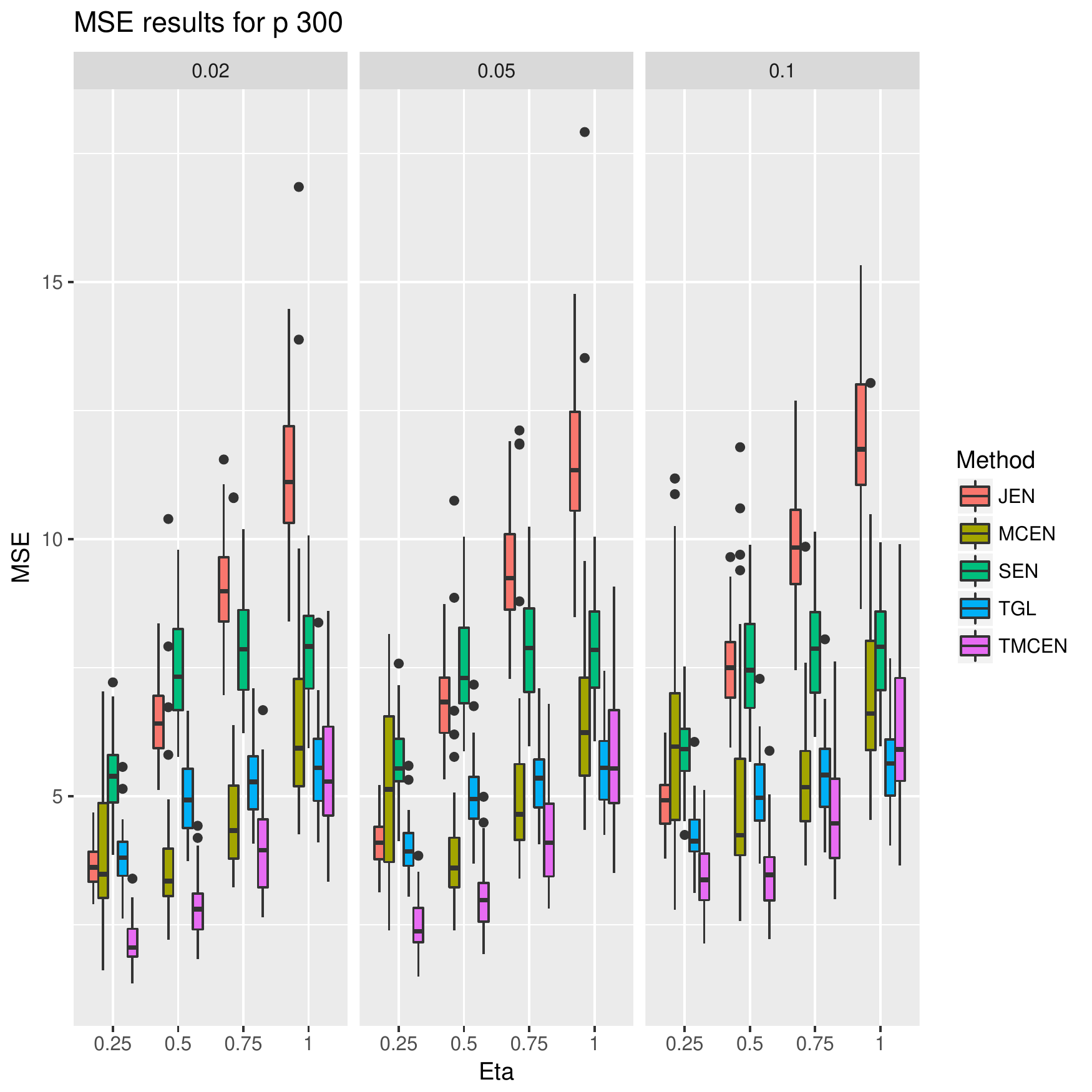}
	\caption{MSE results for the Gaussian simulations with p equal to 300. Different box plots correspond to different values of $\lambda$, while x-axis values are for different values of $\eta$.}
	\label{fig:mse_p_300}
\end{figure}
%
%
%
\begin{figure}
	\centering
		\includegraphics[width=4in]{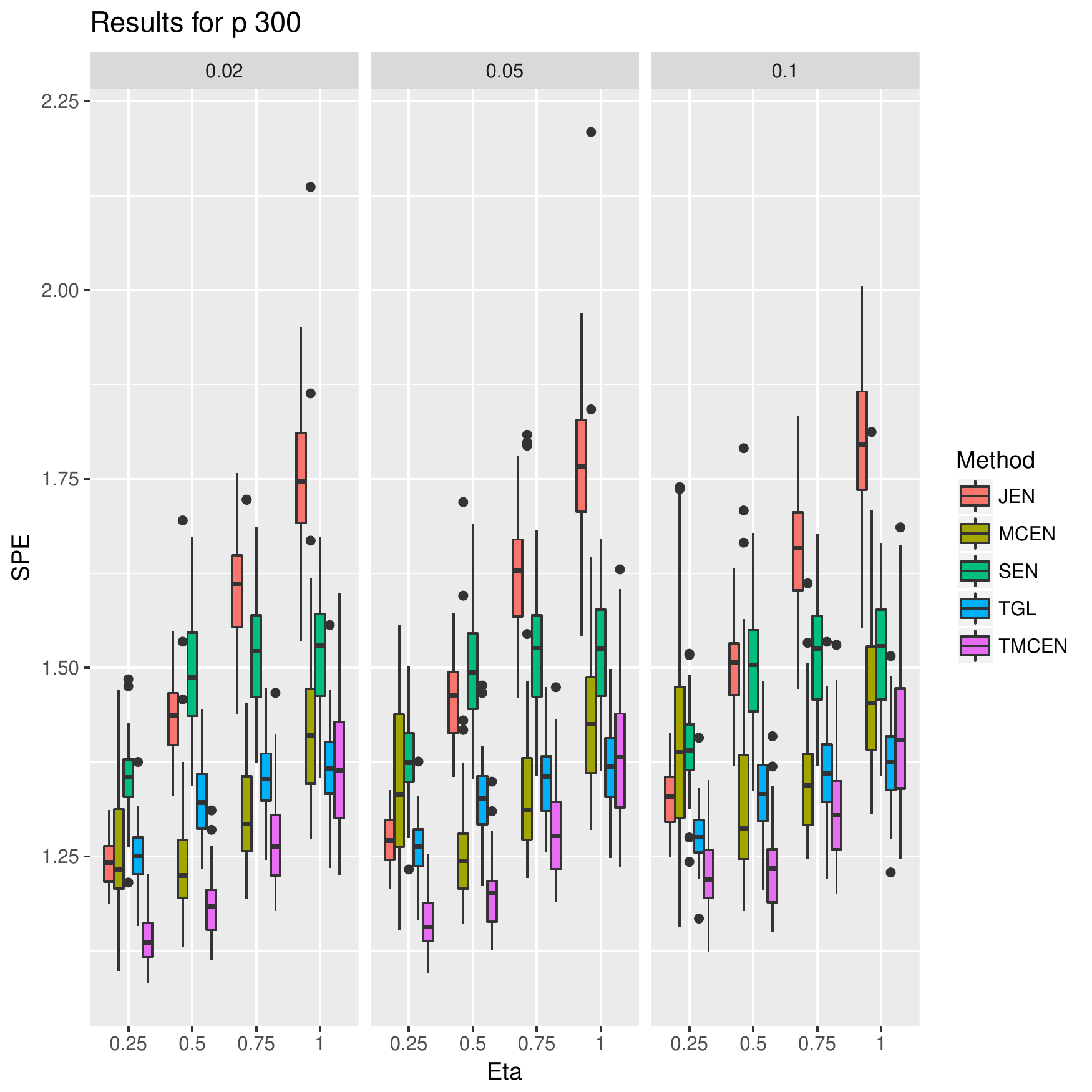}
	\caption{ASPE results for the Gaussian simulations with p equal to 300. Different box plots correspond to different values of $\lambda$, while x-axis values are for different values of $\eta$.}
	\label{fig:pse_p_300}
\end{figure}

\begin{figure}
	\centering
		\includegraphics[width=4in]{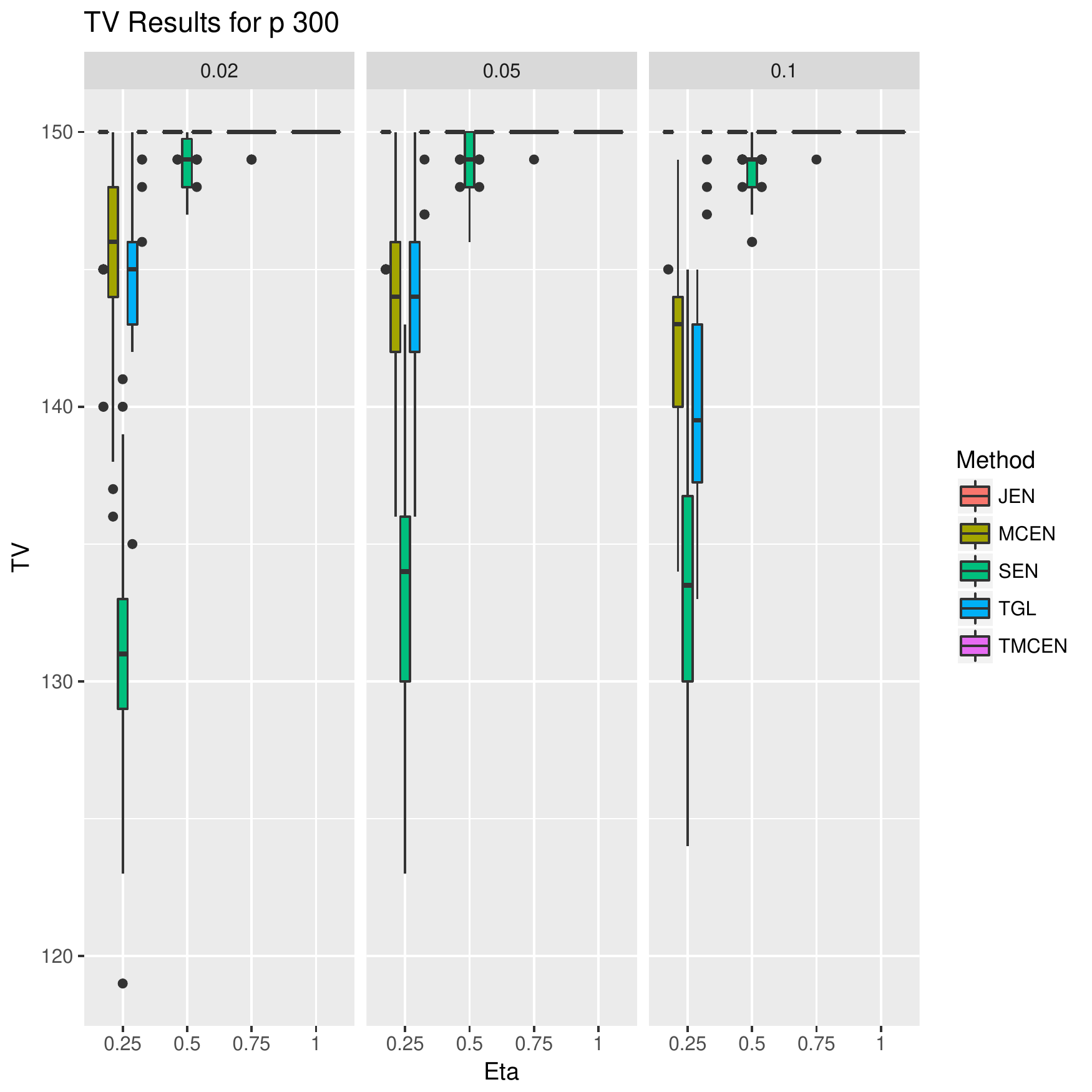}
	\caption{TV results for the Gaussian simulations with p equal to 300. Different box plots correspond to different values of $\lambda$, while x-axis values are for different values of $\eta$.}
	\label{fig:tv_p_300}
\end{figure}

\begin{figure}
	\centering
		\includegraphics[width=4in]{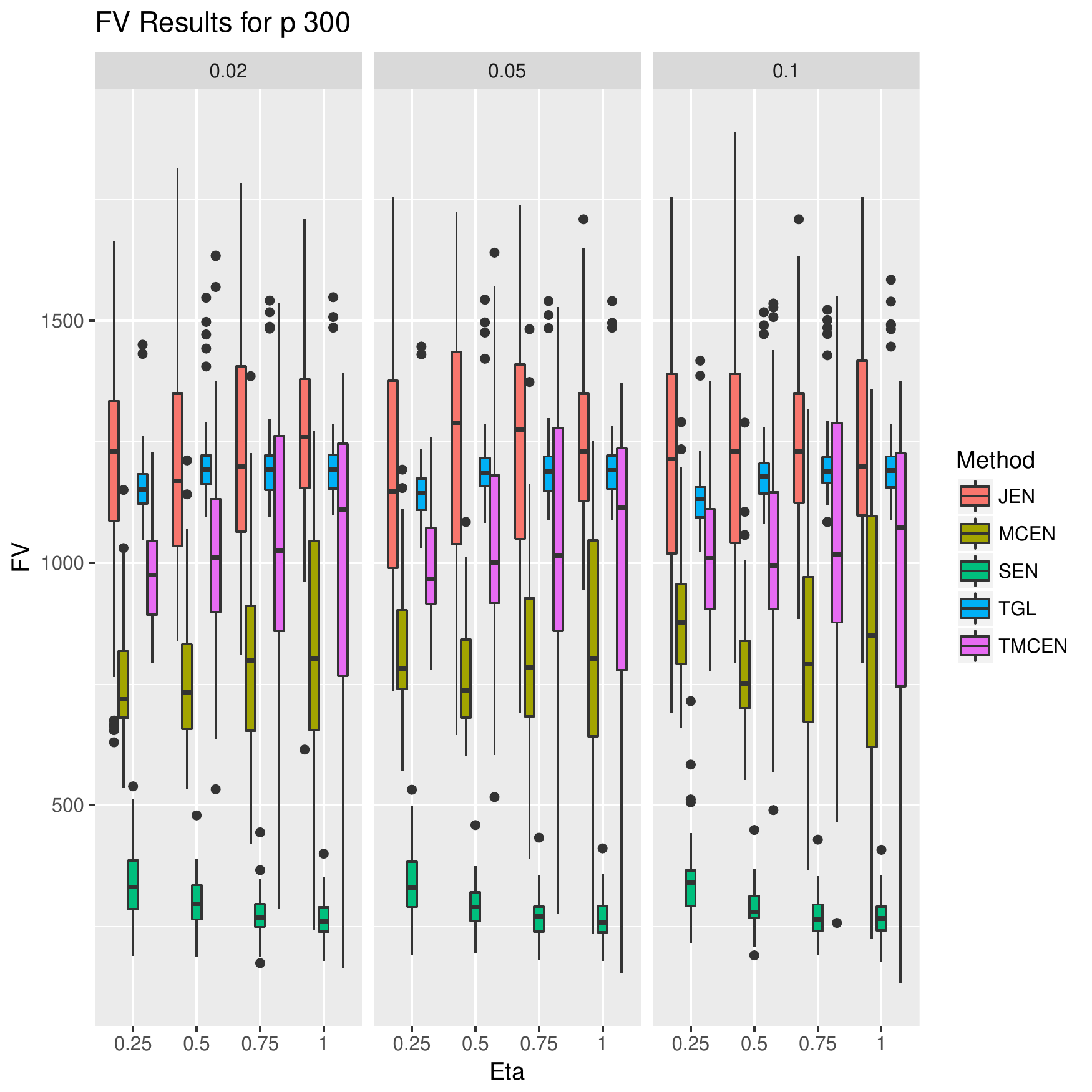}
	\caption{FV results for the Gaussian simulations with p equal to 300. Different box plots correspond to different values of $\lambda$, while x-axis values are for different values of $\eta$.}
	\label{fig:fv_p_300}
\end{figure}

\subsection{Binomial Simulations}

In this setting we have a binomial response variable and compare performance of the MCEN estimator \eqref{BinApprox} to SEN \eqref{glm_init}, for $r=15$ and $p=12,100$ or $300$. The SEN models were fit using the \texttt{glmnet} package in R \citep{friedman2008}. Similar to the previous section, the covariates are generated by $\vx_i \sim N(\mathbf{0}_p,\Sigma_x)$, where $\Sigma_x$ has the same structure provided in the Gaussian simulations with $\rho=.9$. 

We use the same structure of $B$ presented in Section 5.1, consider the same values of $\eta$ and $\lambda$ and again perform 50 replications with a sample size of 100. 
Tuning parameters for the models are estimated via 10-folds cross validation as explained in Section 4.2. For $Q$, the number of groups, we consider values of 2, 3, and 4. For the SEN method each response $c \in \{1,\ldots,r\}$ will be associated with its own tuning parameters of $\gamma_c$ and $\delta_c$ that will be selected by maximizing the equivalent of \eqref{glm_cv} for only one response. 

Define $\vbeta^*_k(\eta,\lambda)$ as the $k$th column vector of $B^*_{\eta,\lambda}$. In all settings the $k$th response of the $i$th observation, $y_{ik}$, is an independent draw from $\mbox{Bin}(1,\pi_{ik}^*)$ where 
\begin{equation*}
\pi_{ik}^* = \frac{\mbox{exp}\left\{\vx_i'\vbeta^*_k(\eta,\lambda)\right\}}{1+\mbox{exp}\left\{\vx_i'\vbeta^*_k(\eta,\lambda)\right\}}.
\end{equation*}
%

%

To evaluate the methods, 1000 validation observations are generated from 
the data generating model and the KL divergence is measured for each of the 50 replications.  The KL divergence for a replication is defined as,

\begin{equation*}
\label{kl}
\frac{1}{1000}\sum_{i=1}^{1000}\sum_{k=1}^{15} \left\{\log\left(\frac{\widehat{\pi}_{ik}}{\pi^*_{ik}}\right)\widehat{\pi}_{ik}+\log\left(\frac{1-\widehat{\pi}_{ik}}{1-\pi^*_{ik}}\right)\left(1-\widehat{\pi}_{ik}\right)\right\}, 
\end{equation*}
where $\pi^*_{ik}$ is the true probability and $\widehat{\pi}_k(x_i,\delta,\gamma)$ is the  estimated probability for response $k$ for validation observation $i$.  

Box plots are presented to compare the KL divergence of MCEN and SEN for the different settings in the case of $p=300$.  The results of simulation in cases where $p=12$ and $100$ are available in the supplementary material. Figure \ref{ind_p300} presents the KL divergence results from the 50 replications for the different settings of $\eta$ and $\lambda$.  In terms of KL divergence MCEN outperforms SEN in all settings.  

A comparison of MSE of coefficient estimates between methods is shown using box plots in Figure \ref{mse_p300} and shows similar results to the cases of $p=12$ and $100$ available in the supplementary material.  The results show that based on MSE binomial MCEN either outperforms or performs as well as binomial SEN.  Figures \ref{bin_tp_p300} and \ref{bin_fp_p300} report the number of true positive and false positive predictors selected by each method for each combination of $\eta$ and $\lambda$ when $p=300$, and MCEN outperforms SEN by generally selecting more true positive predictors, while the number of false predictors selected varies by the signal size. For smaller signals MCEN selects a smaller number of false predictors, but for larger signals MCEN tends to select more false predictors. 




\begin{figure}
\centering
\includegraphics[width=4in]{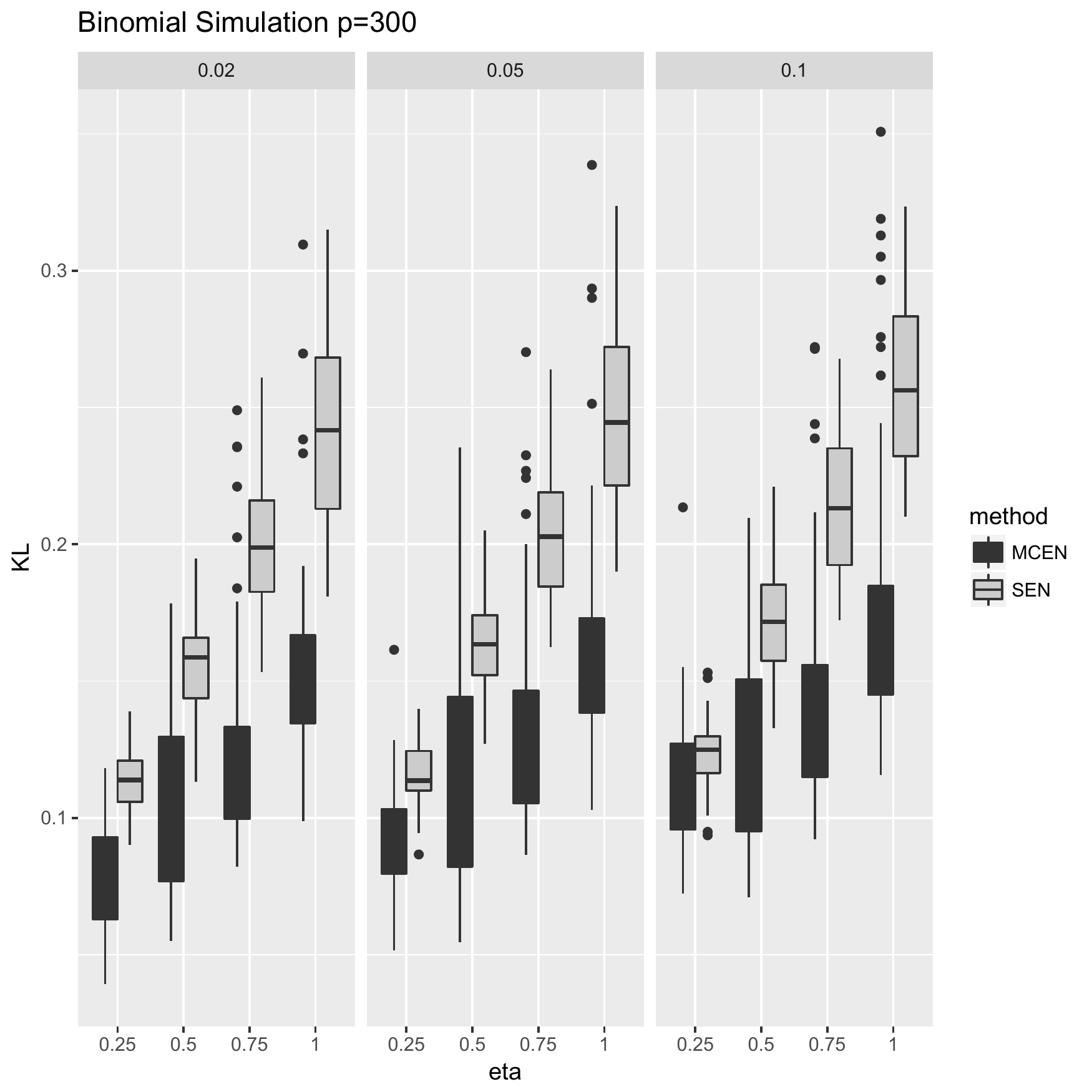}
\caption{Simulation results comparing binomial SEN and binomial MCEN for p=300 at varying values of $\lambda$ and $\eta$. Each box plot represents results
for a different value of $\eta$, given at the top of the plot.}
\label{ind_p300}
\end{figure}

\begin{figure}
\centering
\includegraphics[width=4in]{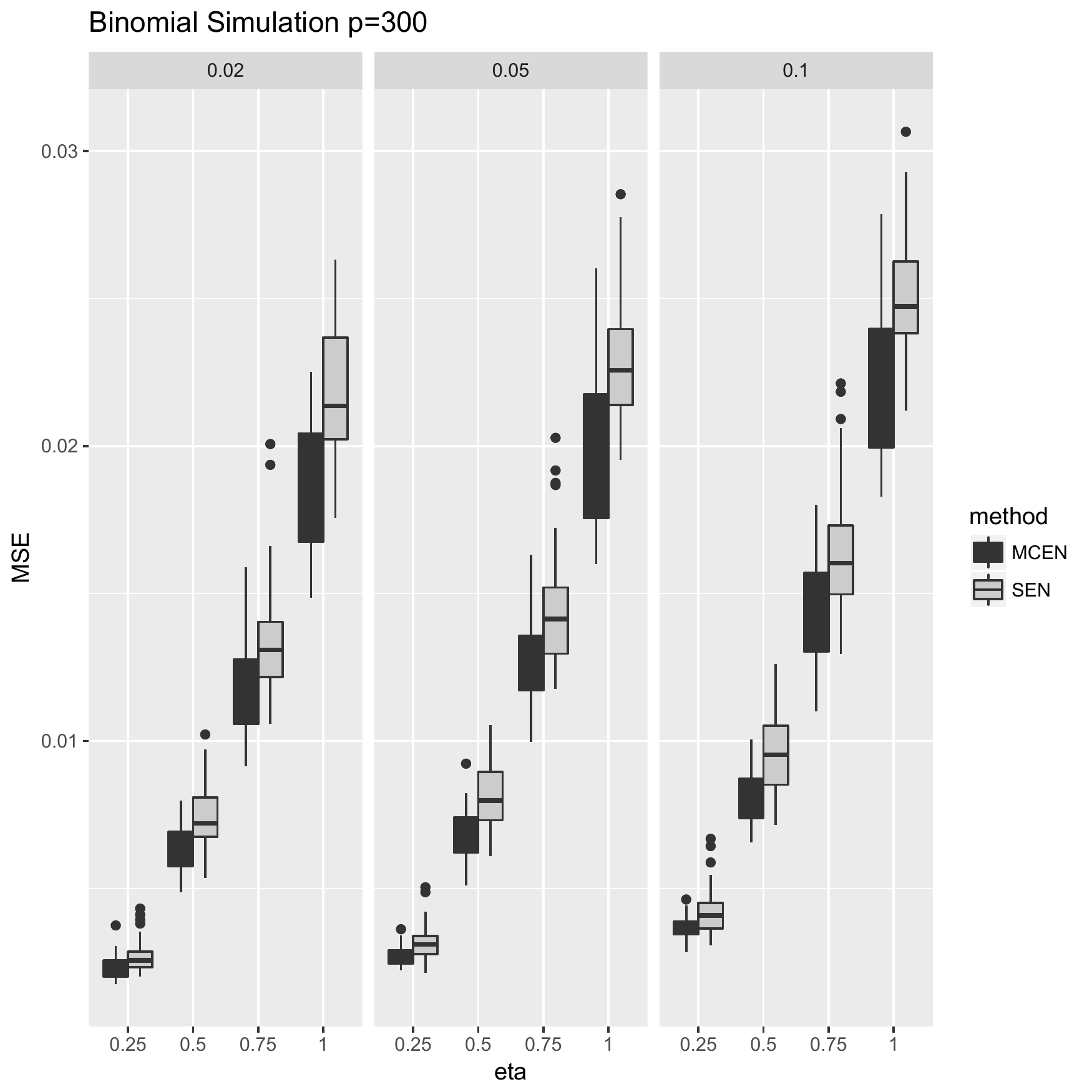}
\caption{Simulation results comparing MSE of binomial SEN and binomial MCEN when p=300 at varying values of $\lambda$ and $\eta$. Each box plot represents results
for a different value of $\eta$, given at the top of the plot.}
\label{mse_p300}
\end{figure}



\begin{figure}
\centering
\includegraphics[width=4in]{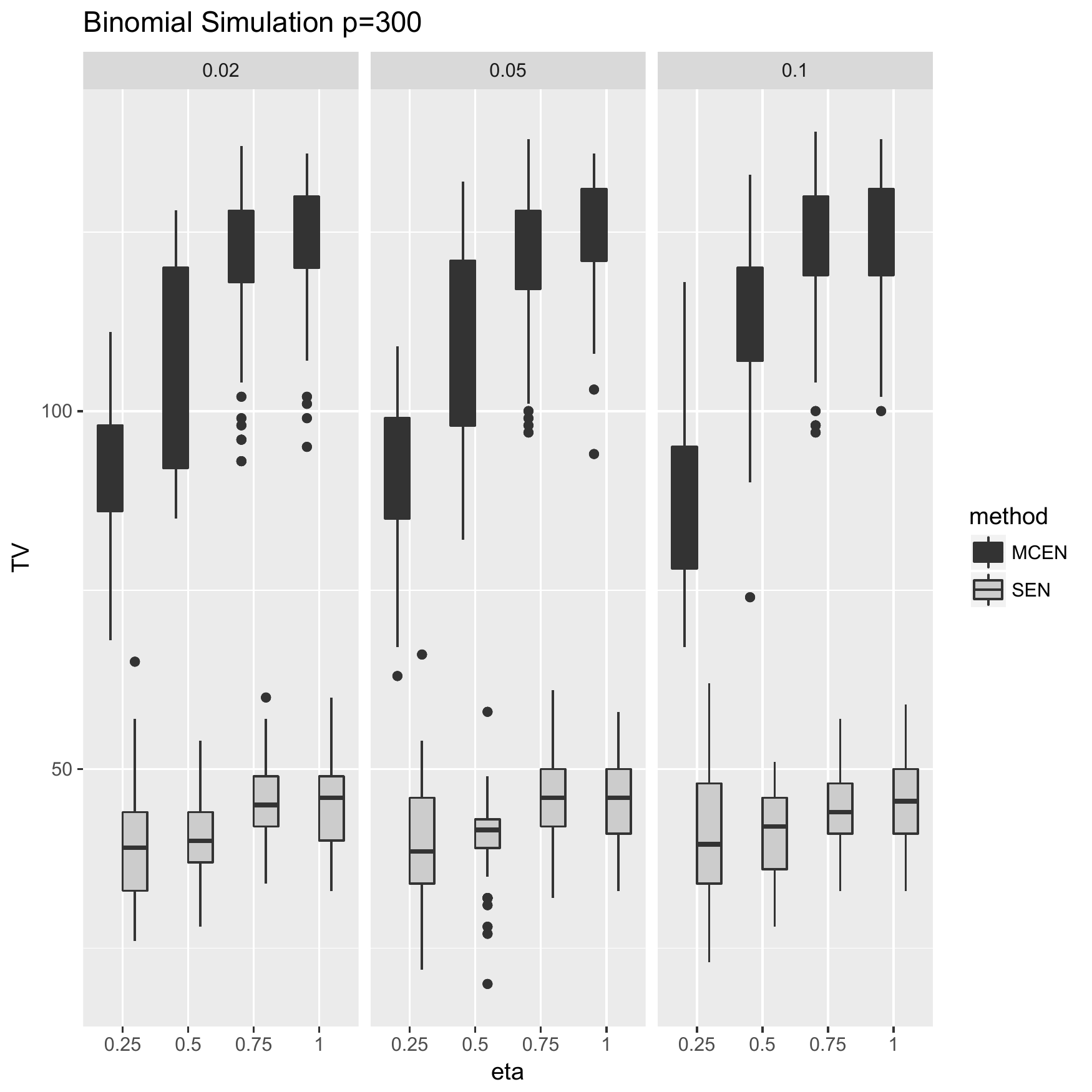}
\caption{Simulation results comparing TV results by binomial SEN and binomial MCEN when p=300 at varying values of $\lambda$ and $\eta$. Each box plot represents results
for a different value of $\eta$, given at the top of the plot.}
\label{bin_tp_p300}
\end{figure}



\begin{figure}
\centering
\includegraphics[width=4in]{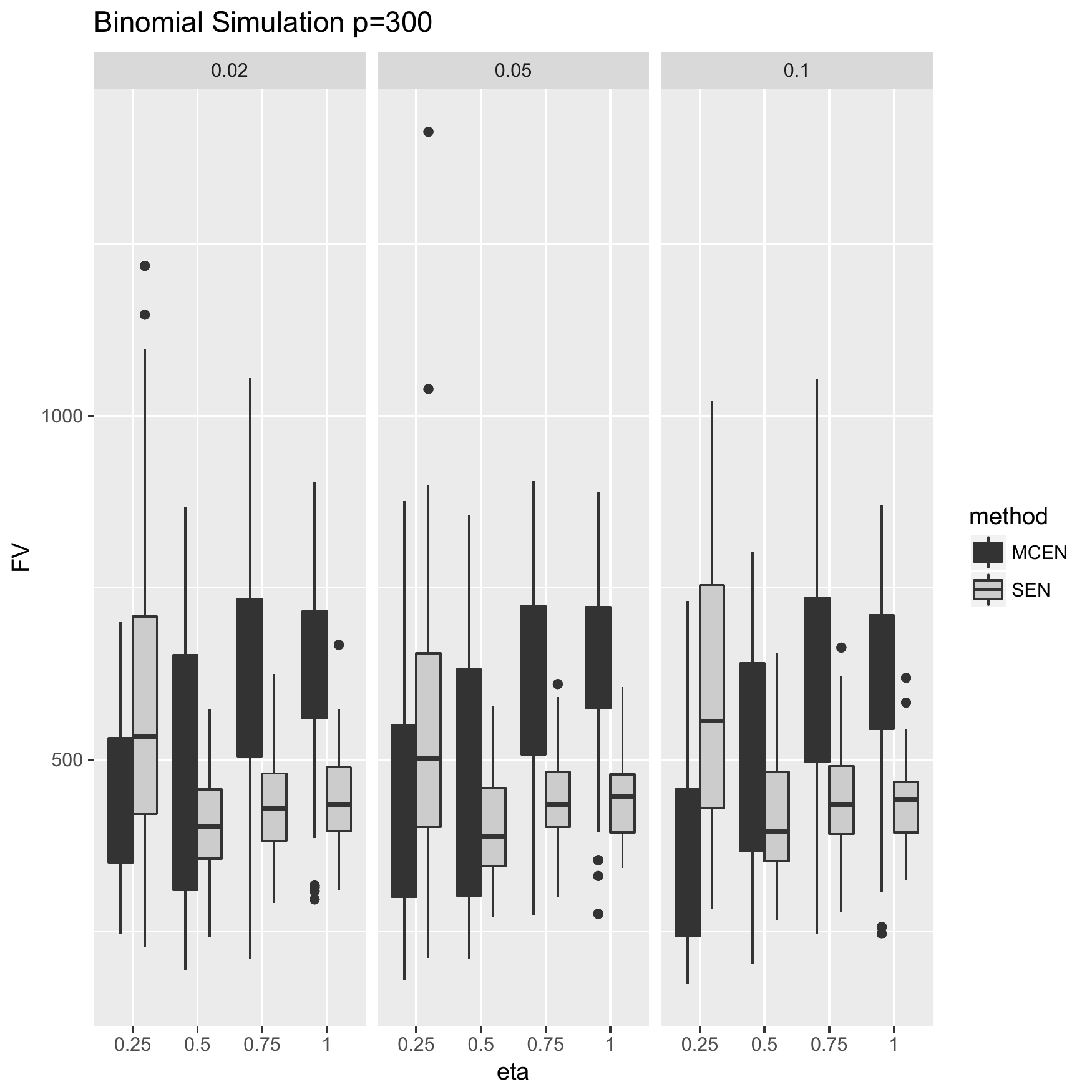}
\caption{Simulation results comparing FV results by binomial SEN and binomial MCEN when p=300 at varying values of $\lambda$ and $\eta$. Each box plot represents results
for a different value of $\eta$, given at the top of the plot.}
\label{bin_fp_p300}
\end{figure}

\section{Data Example}

\subsection{Genomics Data}
\citet{genomicData} collected gene expression profiles, demographic and birth information from 72 pregnant mothers. Using these data we modeled four response variables: placental weight, newborn weight, cotinine level from the mothers' peripheral blood sample and cotinine level from the umbilical cord blood sample. Smoking status, mother's age, mother's BMI, parity, gestational age and expression data for 24,526 gene probes from the mother's peripheral blood sample were used as covariates. Our analysis was limited to the 65 mothers with complete data. From a clinical perspective an accurate model for birth weight would be the primary interest as birth weight is associated with both short and long term negative health outcomes \citep{turan}. Including placental weight as an additional response could potentially be helpful in the MCEN model because previous studies found placental and newborn weight are correlated \citep{linearPlaBrth,nigerianBirthWt,fetalGrowth}, but placental weight is hard to use as a predictor since it is observed at birth. The two measurements of cotinine levels are essentially measuring the same thing and are clearly related to smoking status. Thus we can test if these variables were correctly clustered and smoking status selected in the MCEN models. 

The same methods used in Section 5.1 are used to fit the data, except we did not implement the TMCEN method as we did not assume to know the true clustering structure of the response variables. To evaluate the methods we randomly partitioned the data into 50 training and 15 testing samples. All four response variables are modeled on the log scale. In the training data all variables are centered and scaled to have mean zero and a standard deviation of one. Models are fit using the training data, then predictions are evaluated on the testing samples. We compare the methods by looking at the ASPE, as defined in Section 5.1. For MCEN we consider clusters of size 1, 2 and 3. The process is repeated 100 times and the ASPE for all methods and responses are included in Figure \ref{fig:gene_plot}. The MCEN method performs the best for modeling birth weight, the most clinically interesting variable, and is about the same as the other methods for modeling placenta weight. However, it does worse than the other three methods for modeling cotinine level. In all 100 random partitions the MCEN method correctly grouped the two cotinine responses together and selected smoking status as a predictor for those two responses.  
\begin{figure}
	\centering
		\includegraphics[width=0.9\textwidth]{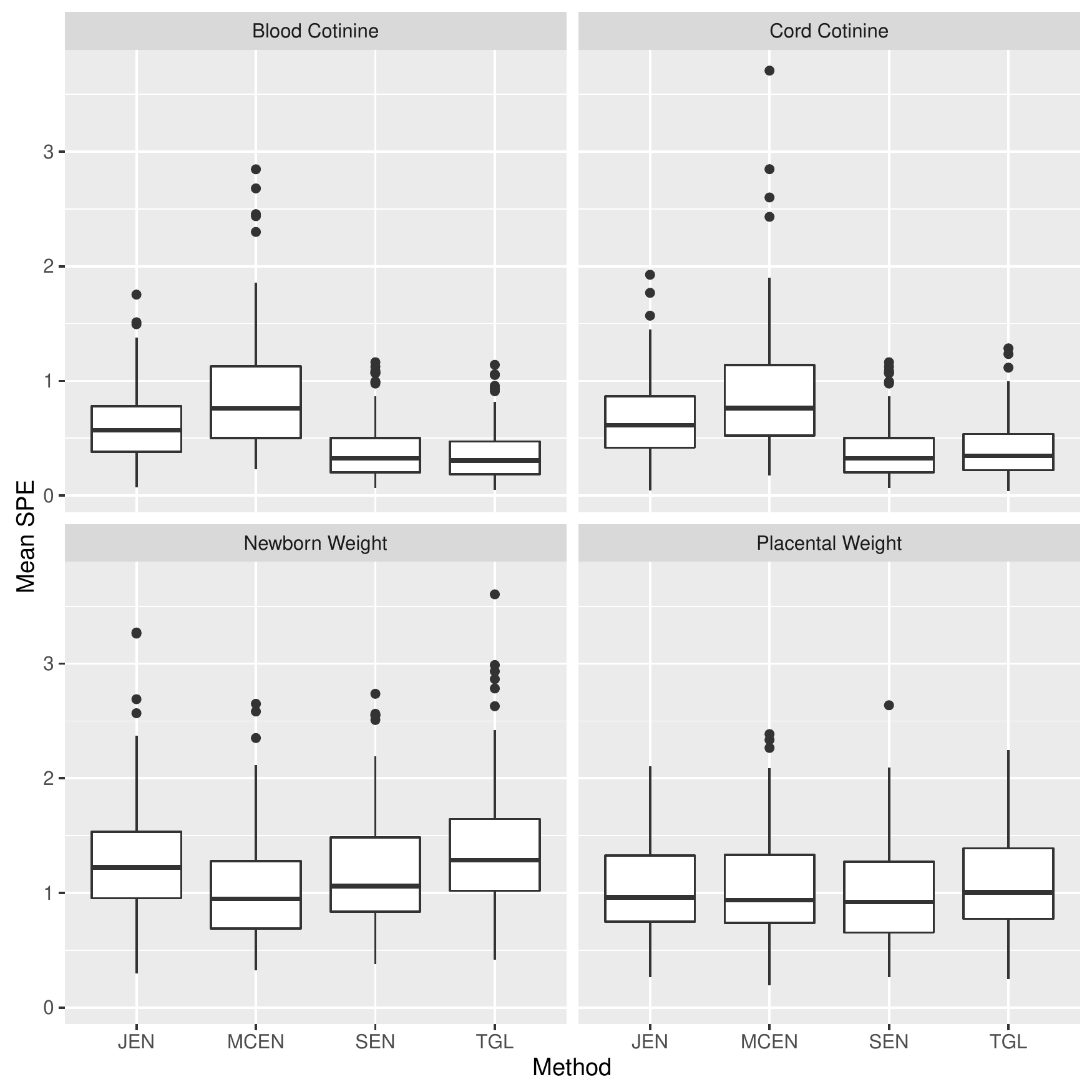}
	\caption{Mean SPE from 100 random partitions}
		\label{fig:gene_plot}
\end{figure}

\subsection{Concession Data}

We analyzed 2000 concession transactions from a major event venue. 
Each transaction is linked with the customer's information from the
venue's loyalty program. These data are proprietary and cannot be made publicly available.  
Whether a customer purchases a specific item, 0 if they do and 1 if they do not, is the response and customer information from the loyalty program, such as seat identification
and amount spent on previous concession sales, are treated as the covariates. The multiple response setting comes from there being multiple items available for sale at the concession stands. In total there are 34 predictor variables, stemming from purchase history from the venue, ticketing, and seating. The same customer may appear in the data more than once, but any correlation structure is ignored. We analyze two different sets of responses with the same covariates. 
The point-of-sale system records purchases in two different item set groupings; menu group (7 items) and food group (12 items).  The different groups provide different insights into customer habits as the items form different groups.   

Similar to the simulation section we compared SEN and MCEN, with tuning parameter selected as described in Section 5.2. For $Q$, the number of groups, we consider values of $\{1,2,\ldots,7\}$.  We divide 2000 transactions into training and validation sets.  
There is a time component to our data, which we ignore, but use to evaluate the predictive performance of our models. The first 1000 transactions are used to train our models, with 3-fold cross validation used to select the tuning parameters for both MCEN and SEN. The predictive performance of the models are then compared using the next 1000 transactions. 
	
For comparison of the methods we present the ROC curves as a metric for classification performance on the 1000 validation observations. Figure \ref{MG_fig} presents the ROC curves and shows that in most situations the binomial logistic MCEN was competitive with SEN.  In this analysis MCEN found 3 response clusters where the first cluster contained concession food, the second cluster contained both alcoholic and non-alcoholic beverages, and the third cluster contained all specialty item groups. For comparison we used k-means clustering on the predicted values of the independent elastic net, and selected the number of clusters based on the gap statistic. It selected 2 clusters. The first cluster had concession and both beverage types, while the second cluster contained all specialty items. 

The resulting ROC curves for the food group analysis are presented in Figure \ref{FG_fig}.  Five clusters were found by binomial logistic MCEN. The first cluster has popcorn, hamburger, french fires, bottled water, appetizers, and a chicken basket.  These correspond to low selling non-alcoholic items.  The second cluster consists of hot dogs, craft beer and misc sides, which represents a group of higher selling items.  The last three clusters are singleton clusters consisting of non-alcoholic beverages, domestic beer, and liquor.  These clusters represent high selling items with different demographics important in each.  We also ran k-means clustering on the predicted values from the EN results, and found no distinct clustering using the gap statistic to select the number of clusters. Thus the MCEN method clusters all cold beverages together, while using k-means on fitted values from SEN does not find this clustering. The results of both analyses show that SEN outperforms MCEN using ROC curves. This could be due to the coarseness of MCEN framework, which assumes a similar sparsity structure for all responses. The grouping insights given from the resulting MCEN clusters provide a starting point for investigating each cluster individually with its own MCEN models. This procedure would allow for different levels of sparsity for different clusters. Flexibility such as this should be addressed in extensions of MCEN. 

\begin{figure}
\begin{center}
\includegraphics[width=5in]{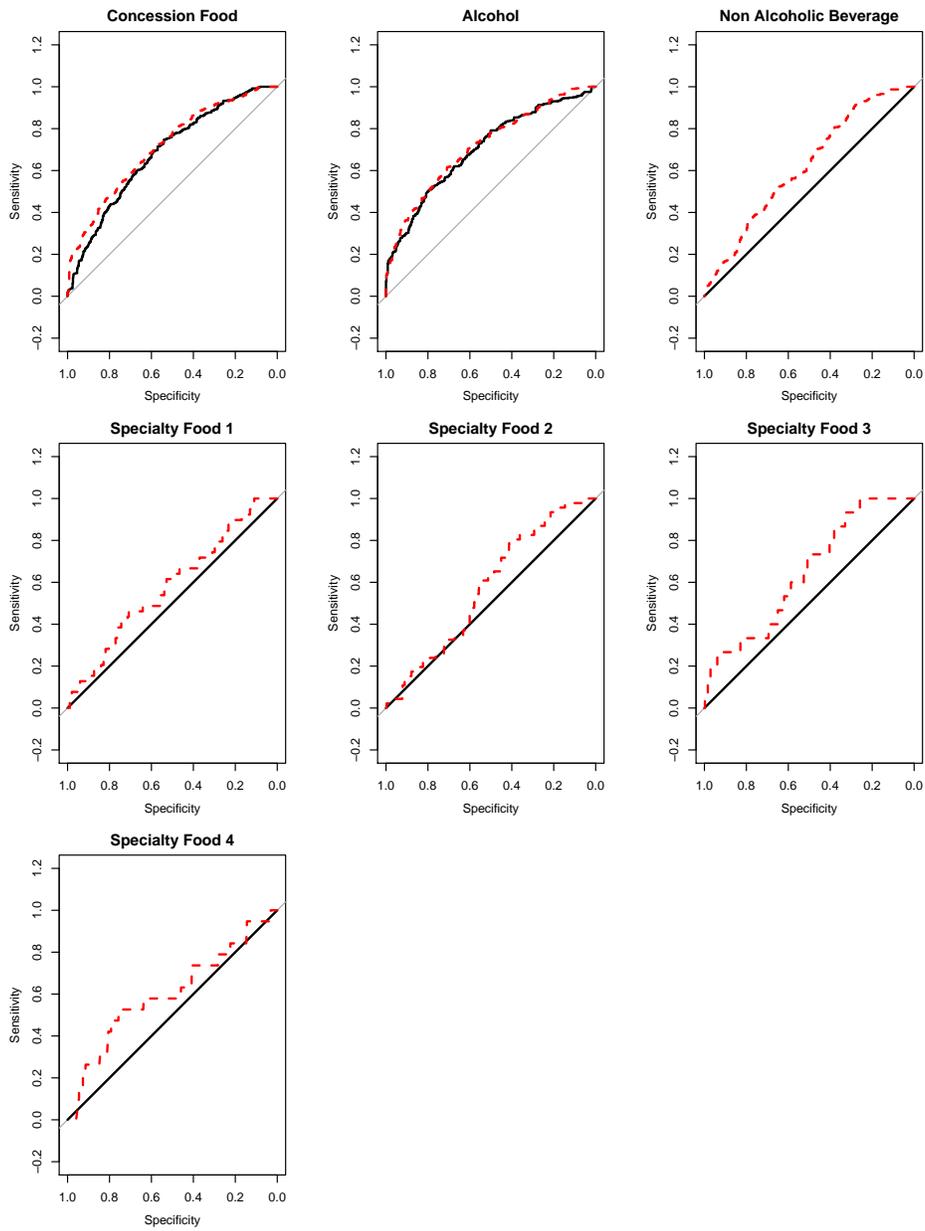}
\end{center}
\caption{ROC curves for the 1000 validation observations for the menu group item responses. The black lines represent the ROC for MCEN and red for SEN. }
\label{MG_fig}
\end{figure}

\begin{figure}
\begin{center}
\includegraphics[width=5in]{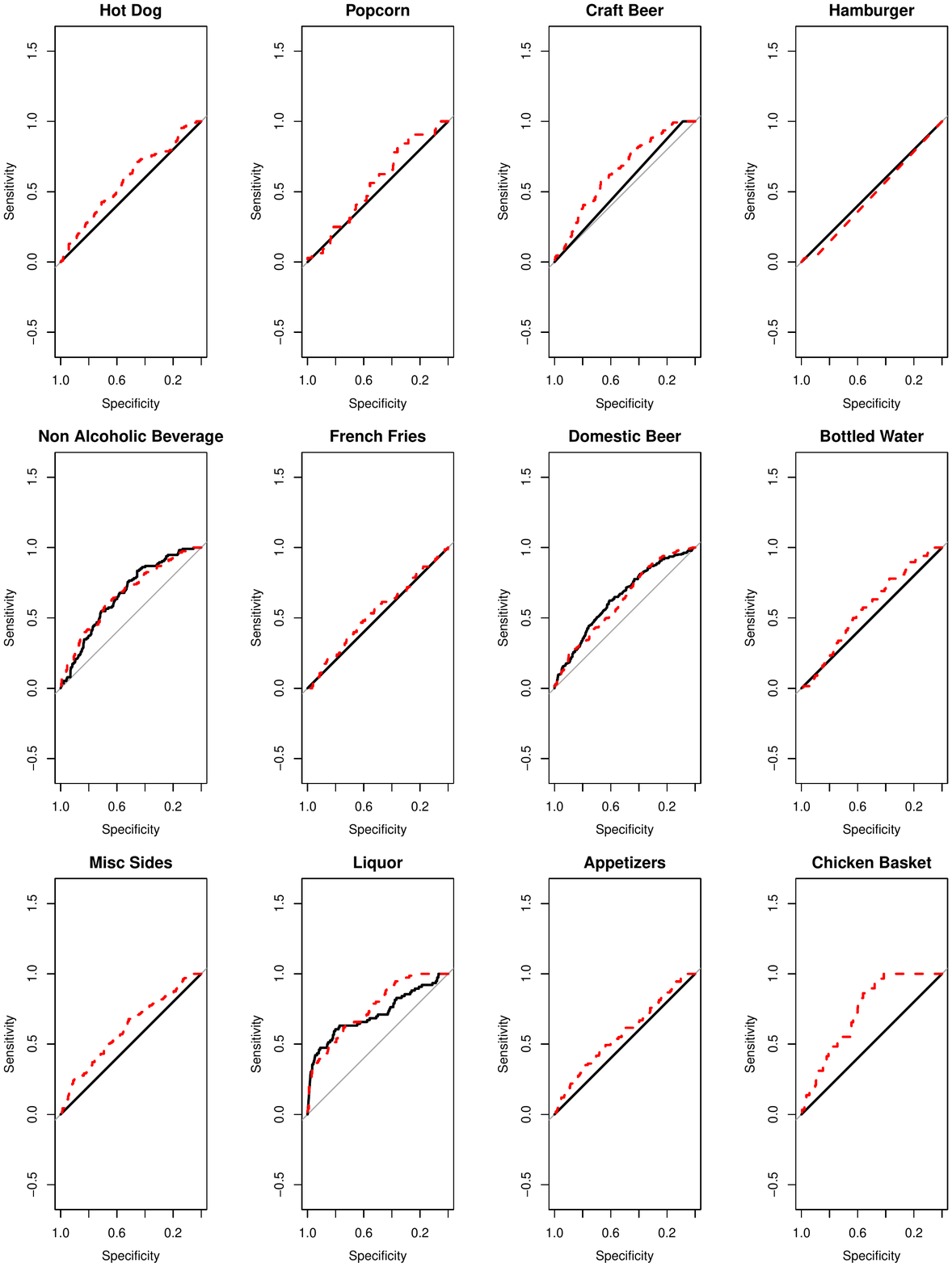}
\end{center}
\caption{ROC curves for the 1000 validation observations for the food group item responses comparing EN and MCEN. The black lines represent the ROC for MCEN and red for SEN.}
\label{FG_fig}
\end{figure}

\section{Discussion}

We present a method for simultaneous estimation of regression coefficients and response clustering for a multivariate response model. The method is introduced for the case of continuous and binary responses. Future work could include extending the model to other GLM settings. Currently, our model imposes the same amount of sparsity on all response models, but this could be relaxed by allowing a sparsity tuning parameter for each individual response or each response group. An R package that implements the methods outlined in this article will be available on CRAN, upon publication of this work.

Define $\ell(B)$ as a likelihood or convex objective function, $P(\vbeta,D_q)$ as a distance function between all elements where $\sum_{q=1}^Q P(B,D_q)$ is an optimization problem to separate the $r$ $p$-dimensional coefficient vectors into $Q$ clusters and $p_\delta(B)$ as a penalty function with tuning parameter $\delta$. Then the MCEN method could be generalized to a larger class of estimators where
\begin{equation}
\label{gen_opt}
(\hat{B},\hat{D}) = \argmin_{B, D_1,\ldots, D_Q} \ell(B) + \gamma \sum_{q=1}^Q P(B,D_q) + p_\delta(B).
\end{equation}
One example would be to define $P(B,D_q)$ as an $L_1$ norm to penalize the difference between fitted values, similar to a fused lasso penalty \citep{tibs05, tibs_14}. 
An advantage of the estimator proposed in this paper is that by defining $P(B,D_q)$ as the $L_2$ norm squared, when the coefficients are fixed, the minimization problem is equivalent to a k-means problem. However, different definitions of $P(B,D_q)$ may not have well studied clustering algorithms to solve the optimization to define the groupings. One challenge of extending this work would be finding functions $P(B,D_q)$ that become well defined clustering problems when $B$ is known or proposing new algorithms for solving $P(B,D_q)$. Otherwise the two-step algorithm proposed in this paper would not work.


The asymptotics in this paper are limited to consistency of the estimator when groups are known. \citet{zhao_16} presented an inference framework for a similar estimator that uses a fusion penalty and demonstrated that inference is still possible even if the structure of the graph that determines the fusion penalty is not correctly specified. Extending the results provided here to include inference would be of great use to practitioners and a good topic for future research.

\section*{Appendix}
\section*{A.1. Proof of Theorem \ref{thm1}} 
\begin{proof}
Define 
\begin{equation*}
L(B) = \frac{1}{2n}\sum_{i=1}^n\sum_{c=1}^r (y_{ic} -\vx_i^T\vbeta_c)^2 +\frac{\gamma}{2n}\sum_{q=1}^Q\frac{1}{|D_q|}\sum_{l,m \in D_q} ||X(\vbeta_l-\vbeta_m)||_2^2.
\end{equation*}
For $l \in D_q$
\begin{equation*}
\frac{\partial}{\partial \vbeta_l} L(B) = -\frac{1}{n} \left(X^TY - X^TX\vbeta_l\right) + X^TX \frac{2\gamma}{n|D_q|} \sum_{c \in D_q, \, c \neq l}  \vbeta_l - \vbeta_c.
\end{equation*}
Thus,
\begin{equation}
\label{thm1_eq1}
\bar{\vbeta}_l\left\{1 + \frac{2\gamma(|D_q|-1)}{|D_q|}\right\}  - \dot{\vbeta}_l  - \frac{2\gamma}{|D_q|} \sum_{c \in D_q, \, c \neq l} \bar{\vbeta}_c = 0.
\end{equation}
Therefore for $l,m \in D_q$ 
\begin{eqnarray*}
&&  \bar{\vbeta}_l\left\{1 + \frac{2\gamma(|D_q|-1)}{|D_q|}\right\}  - \dot{\vbeta}_l  - \frac{2\gamma}{|D_q|} \sum_{c \in D_q, \, c \neq l} \bar{\vbeta}_c \\
&-& \bar{\vbeta}_m\left\{1 + \frac{2\gamma(|D_q|-1)}{|D_q|}\right\}  - \dot{\vbeta}_m  - \frac{2\gamma}{|D_q|} \sum_{c \in D_q, \, c \neq m} \bar{\vbeta}_c \\
&=& \left(\bar{\vbeta}_l-\bar{\vbeta}_m\right)\left(1 + 2\gamma\right)  - \dot{\vbeta}_l + \dot{\vbeta}_m = 0.
\end{eqnarray*}
Therefore for $l,m \in D_q$ and $l \neq m$
\begin{equation}
\label{thm1_eq2}
\bar{\vbeta}_m = \bar{\vbeta}_l + \frac{1}{1+2\gamma}\left(\dot{\vbeta}_m - \dot{\vbeta}_l\right).
\end{equation}
Combining \eqref{thm1_eq1} and \eqref{thm1_eq2} gives
\begin{eqnarray*}
\bar{\vbeta}_l\left\{1 + \frac{2\gamma(|D_q|-1)}{|D_q|}\right\} &=& \dot{\vbeta}_l  + \frac{2\gamma}{|D_q|} \sum_{c \in D_q, \, c \neq l} \bar{\vbeta}_l + \frac{1}{1+2\gamma}\left(\dot{\vbeta}_c - \dot{\vbeta}_l\right)\\
&=& \dot{\vbeta}_l  + \frac{2\gamma(|D_q|-1)}{|D_q|}  \bar{\vbeta}_l + \frac{2\gamma}{(1+2\gamma)|D_q|}\sum_{c \in D_q, \, c \neq l}\left(\dot{\vbeta}_c - \dot{\vbeta}_l\right),
\end{eqnarray*}
which completes the proof. 
\end{proof}

\section*{A.2. Proof of Theorem \ref{thm2}}
\begin{proof}
It is assumed that $E(\epsilon_{ic}^2) = 1$ and for $c \neq k$ that $E(\epsilon_{ic}\epsilon_{ik})=\rho$. Thus,
note that for any $v \in \{1,\ldots,r\}$  
\begin{eqnarray*}
\mbox{Var}\left(\bar{\vbeta}_v\right) &=& \mbox{Var}\left\{\frac{|D_q|+2\gamma}{(1+2\gamma)|D_q|}\dot{\vbeta}_v+\frac{2\gamma}{(1+2\gamma)|D_q|}\sum_{s \in D_q, s\neq v}\dot{\vbeta}_s\right\} \\
&=& (X^TX)^{-1}\Biggl\{\frac{|D_q|\left(|D_q|+4\gamma+4\gamma^2\right)}{(1+2\gamma)^2|D_q|^2}  + 4\rho \gamma (|D_q|-1) \frac{|D_q|+2\gamma|D_q|-2\gamma}{(1+2\gamma)^2|D_q|^2}\Biggr\}.
\end{eqnarray*}

Define $\vb_v = \sum_{s \in D_q, s\neq v} \left(\vbeta_{s}^* - \vbeta_{v}^*\right)$. The squared bias term is then 
\begin{eqnarray*} 
&& E\left[\left\{E\left(\bar{\vbeta}_v\right)-\vbeta_{v}^*\right\}'\left\{E\left(\bar{\vbeta}_v\right)-\vbeta_v^*\right\}\right]\\
&=& E\left[\left\{\vbeta_v^* + \frac{2\gamma}{(1+2\gamma)|D_q|}\vb_v -\vbeta_v^*\right\}'\left\{\vbeta_v^* + \frac{2\gamma}{(1+2\gamma)|D_q|}\vb_v-\vbeta_v^*\right\}\right] \\
&=& \frac{4\gamma^2}{(1+2\gamma)^2|D_q|^2}||\vb_v||_2^2.
\end{eqnarray*}
Let $\omega = \mbox{Trace}\left\{\left(X^TX\right)^{-1}\right\}$ then MSE of $\bar{\vbeta}_v$ will be smaller than MSE of $\dot{\vbeta}_v$ if 
\begin{eqnarray*}
&& \omega\left\{\frac{|D_q|\left(|D_q|+4\gamma+4\gamma^2\right)}{(1+2\gamma)^2|D_q|^2}  + 4\rho \gamma (|D_q|-1) \frac{|D_q|+2\gamma|D_q|-2\gamma}{(1+2\gamma)^2|D_q|^2}\right\} \\
&+& \frac{4\gamma^2}{(1+2\gamma)^2|D_q|^2}||\vb_v||_2^2 \\
&<& \omega, 
\end{eqnarray*}
which is equivalent to
\begin{eqnarray*}
\gamma ||\vb_v||_2^2 &<& \omega \Biggl\{ |D_q|(|D_q|-1) + \gamma |D_q|(|D_q|-1)  \\
&& - \rho\left\{ (|D_q|-1)|D_q| + 2\gamma(|D_q|-1)^2\right\}\Biggr\}.
\end{eqnarray*}
Note that, $\omega|D_q|(|D_q|-1)(1-\rho)>0$ and thus if $||\vb_v||_2^2 \leq \omega (|D_q|-1)\left\{|D_q|-2\rho(|D_q|-1)\right\}$ then
the MSE of $\bar{\vbeta}_v$ is smaller than the MSE of $\dot{\vbeta}_v$ for any $\gamma > 0$. Otherwise, the MSE of $\bar{\vbeta}_v$ will be
smaller for any $\gamma \in \left(0, \frac{\omega|D_q|(|D_q|-1)(1-\rho)}{||\vb_v||_2^2 - \omega(|D_q|-1)\left\{|D_q|-2\rho(|D_q|-1)\right\}}\right)$. Thus for any $v \in \{1,\ldots,r\}$ then any $\gamma >0$ or any $\gamma$ sufficiently small will result in $\bar{\vbeta}_v$ having a smaller MSE than $\dot{\vbeta}_v$. The proof is complete because we can then find a $\gamma$ sufficiently small that will result in $\bar{\vbeta}_v$ having a smaller MSE than $\dot{\vbeta}_v$ for all $v \in \{1,\ldots,r\}$.
\end{proof}

\section*{A.3. Proof of Corollary \ref{cor1}}

The proof of Corollary \ref{cor1} is similar to the proof of Theorem \ref{thm1} and only changes with respect to the expected loss rather than the observed loss.

\section*{A.4. Theorem \ref{thm_strong}}
The proof of Theorem \ref{thm_strong} will include some new definitions and an alternative formulation of \eqref{opt_dfixed}. In our proof we use a vectorized version of many of the matrices. Let $\tilde{\vY} = \mbox{vec}(Y)$, $\tilde{\vbeta} = \mbox{vec}(B)$, $\tilde{\vbeta}' = \mbox{vec}(\acute{B})$ and $\tilde{\vE} = \mbox{vec}(E)$. Define $\vA_{m,s} \in \Real^r$, where $(m,s) \in D_q$, with $\sqrt{\frac{1}{|D_q|}}$ in the $m$th element, $-\sqrt{\frac{1}{|D_q|}}$ in the $s$th element and 0 in all other elements, $A_{D_q} \in \Real^{|D_q|(|D_q|-1) \times r}$ as the matrix with row vectors $\vA_{m,s}$ where $(m,s) \in D_q$, and $A_D \equiv \left(A_{D_1}^T,\ldots, A_{D_Q}^T\right)^T \in \Real^{\sum_{q=1}^Q |D_q|(|D_q|-1)\times r}$. 

Then the objective function from \eqref{opt_dfixed} can be restated as 
\begin{eqnarray*}
&& \frac{1}{2n} \left[ \tilde{\vbeta}^T\left\{ \tilde{X}^T\tilde{X}+\gamma(A_D\otimes X)^T(A_D \otimes X)\right\} \tilde{\vbeta} -2\tilde{\vY}^T\tilde{X}\tilde{\vbeta}\right] +\delta ||\tilde{\vbeta}||_1 \\
&=& \ell(\tilde{\vbeta})+\delta g(\tilde{\vbeta}).
\end{eqnarray*}
In addition define, $\tilde{\ell}(\vDelta,\tilde{\vbeta}) \equiv \ell(\tilde{\vbeta}+\vDelta) - \ell(\tilde{\vbeta}) - \langle \nabla \ell(\tilde{\vbeta}, \Delta) \rangle$.

First, we will present some lemmas that are helpful in proving Theorem \ref{thm_strong}. A general outline of the proof for Theorem \ref{thm_strong} is by using the triangle inequality we have $||\mbox{vec}(\bar{B}-B*)||_2 \leq ||\mbox{vec}(\bar{B}-\acute{B})||_2 + ||\tilde{\vbeta}'-\tilde{\vbeta}^*||_2$. Completing the proof is done by establishing upper bounds for $||\mbox{vec}(\bar{B}-\acute{B})||_2$ and $||\tilde{\vbeta}'-\tilde{\vbeta}^*||_2$. Much of the proof will require working with $\tilde{\vbeta}'$ and we introduce the following notation to easily relate $\tilde{\vbeta}'$ and $\tilde{\vbeta}^*$. For response $l$ in group $q$ define $\vH_l = \frac{1}{\sqrt{|D_q|}} \sum_{c \in D_q, c \neq l} \vA_{c,l}$ where $\vH_l \in \Real^r$ and 
$H = \left(\vH_1,\ldots,\vH_r\right)^T \in \Real^{r \times r}$. Then we have
\begin{equation*}
\tilde{\vbeta}' = \left\{\left(I_r +\frac{2\gamma}{2\gamma+1}H\right)\otimes I_p \right\} \tilde{\vbeta}^*.
\end{equation*}
For response $l$ is in group $q$ define $\vU_l = \frac{1}{|D_q|}\sum_{ k \in D_q} (\beta^*_k-\beta^*_l)$ where $\vU_l \in \Real^{p}$ and $U = \left(\vU_1,\ldots,\vU_r\right)$ with $U \in \Real^{p \times r}$ and $\tilde{\vU}=\mbox{vec}(U) \in \Real^{pr}$, then 
\begin{equation*}
\left\| \mbox{vec}\left(\acute{B}-B^*\right) \right\|_2 = \frac{2\gamma}{1+2\gamma} \left\| (H\otimes I_p)\tilde{\vbeta}^* \right\|_2 = \frac{2\gamma}{1+2\gamma} \left\| \tilde{\vU}\right\|_2.
\end{equation*}


\begin{lemma}
\label{restrictedStrong}
Under assumption A3 
\begin{equation*}
\tilde{\ell}(\vDelta,\tilde{\vbeta}') \geq \kappa ||\vDelta||_2^2 \mbox{ for all } \vDelta \in \mathcal{C}. 
\end{equation*}
\end{lemma}
\begin{proof}
From the definition of $\tilde{\ell}(\vDelta,\tilde{\vbeta})$, assumption A3 and that $\vDelta \in \mathcal{C}$ it follows that
\begin{eqnarray*}
\tilde{\ell}(\vDelta,\tilde{\vbeta}')  &=& \frac{1}{2n} \vDelta^T\left\{ \tilde{X}^T\tilde{X}+\gamma(A_D\otimes X)^T(A_D \otimes X)\right\}\vDelta \\
&\geq& \frac{1}{2n} \vDelta^T \tilde{X}^T\tilde{X} \vDelta \\
&\geq& \frac{\kappa}{2} ||\vDelta||_2^2.
\end{eqnarray*}
\end{proof}

For any vector $\va = (a_1,\ldots,a_{pr})^T \in \Real^{pr}$ we define the $||\va||_{\infty}$ as the $L_\infty$ norm of $\va$, that is $||\va||_{\infty} = \underset{i}{\max} |a_i|$.

\begin{lemma}
\label{deltaLemma}
For $\bar{B}$ from \eqref{opt_dfixed}, 
under assumptions A1-A4 with $\delta \geq 2\left|\left|\nabla \ell(\tilde{\vbeta}')\right|\right|_{\infty}$ then there exists a positive constant $c_3$ such that 
\begin{equation*}
\left|\left| \mbox{vec}(\bar{B}-\acute{B})\right|\right|_2^2 \leq 9 \frac{\delta^2}{\kappa^2} s.
\end{equation*}
\end{lemma}
\begin{proof}
Define the set $\acute{S} = \{ j \in \{1,\ldots,rp\}, \tilde{\vbeta}'_j \neq 0\}$. By assumption A5 and Corollary \ref{cor1} $\acute{S}=S$, that is $\tilde{\vbeta}'_j = 0$ if and only if $\tilde{\vbeta}^*_j=0$. Define $\psi(\mathcal{M}) \equiv \underset{\vu \in \mathcal{M}\setminus \{\mathbf{0}\}}{\mbox{sup}} \frac{||\vu||_1}{||\vu||_2}$. Note that $\psi\{\mathcal{M}(S)\} = \sqrt{s}$.  Also, note that the dual norm of the $L_1$ norm is the $L_\infty$ norm. Results follow from Theorem 1 of \citet{negahban2012} and Lemma \ref{restrictedStrong}. 
\end{proof}


\begin{lemma}
\label{acute_beta_lemma}
Under the conditions of Theorem \ref{thm_strong} there exists positive $c_1$, $c_2$ and $c_3$ such that 
\begin{equation*}
\left|\left|\mbox{vec}\left(\bar{B}-\acute{B}\right)\right|\right|_2 \leq \frac{48\sigma}{\kappa}\sqrt{\frac{s\log(rp)}{n}},
\end{equation*}
with probability at least $1-c_1\mbox{exp}(-c_2n\delta^2)$.
\end{lemma}

\begin{proof}
If we can find positive constants $c_1$ and $c_2$ such that with probability at least $1-c_1\mbox{exp}(-c_2n\delta^2)$ that $\delta \geq 2\left|\left|\nabla \ell(\tilde{\vbeta}')\right|\right|_{\infty}$ then proof will be complete by Lemma \ref{deltaLemma} and by the condition that $\delta =  16\sigma\sqrt{\frac{\log(rp)}{n}}$. Note that  
\begin{eqnarray*}
2\left|\left|\nabla \ell(\tilde{\vbeta}')\right|\right|_{\infty} &=& 2 \left|\left| \frac{1}{n}\left[\left\{ \tilde{X}^T\tilde{X}+\gamma(A_D\otimes X)^T(A_D \otimes X)\right\}\tilde{\vbeta}'  - \tilde{X}^T\tilde{Y} \right] \right|\right|_{\infty} \\
&=& 2 \left|\left| \frac{1}{n}\left[\left\{ \tilde{X}^T\tilde{X}+\gamma(A_D\otimes X)^T(A_D \otimes X)\right\}\left\{\left(I_r +\frac{2\gamma}{2\gamma+1}H\right)\otimes I_p \right\} \tilde{\vbeta}^*  - \tilde{X}^T\left(\tilde{X}\tilde{\vbeta}^*+\tilde{\vE}\right) \right] \right|\right|_{\infty} \\
&\leq& 2 \left\|\frac{2\gamma}{n(1+2\gamma)}\tilde{X}^T\tilde{X}\tilde{\vU}\right\|_{\infty} + 2\left\|\frac{\gamma}{n}(A_D\otimes X)^T(A_D\otimes X)\tilde{\vbeta}^*\right\|_{\infty}\\
&& +2\left\|\frac{2\gamma^2}{n(1+2\gamma)}(A_D\otimes X)^T(A_D\otimes X)\tilde{\vU}\right\|_{\infty} + 2\left\|\frac{1}{n}\tilde{X}^T\tilde{\vE}\right\|_{\infty}.
\end{eqnarray*}
Next, we will establish upper bounds for the first three terms. Define $I(l \in D_q)$ to be 1 if $l \in D_q$ and zero otherwise. 
Using the definition of $\tilde{\vU}$ and assumptions A4-A6,
\begin{eqnarray*}
2\left\| \frac{2\gamma}{n(1+2\gamma)} \tilde{X}^T\tilde{X}\tilde{\vU}\right\|_\infty &=& \frac{4\gamma}{1+2\gamma}\max_{l\in\{1,\ldots,r\}} \left\|\frac{1}{n}X^TX\sum_{q=1}^QI(l \in D_q)\sum_{k \in D_q} \frac{1}{|D_q|}(\vbeta^*_k-\vbeta^*_l) \right\|_\infty\\
&\leq&  \frac{4\gamma}{1+2\gamma}\rho_{\max} \max_{l\in\{1,\ldots,r\}} \left\|\sum_{q=1}^QI(l \in D_q)\sum_{k \in D_q} \frac{1}{|D_q|}(\vbeta^*_k-\vbeta^*_l) \right\|_2\\
&\leq& \frac{4\gamma}{1+2\gamma}\rho_{\max}\acute{b}.
\end{eqnarray*}
Using assumptions A4-A6,
\begin{eqnarray*}
2\left\| \frac{\gamma}{n}(A_D\otimes X)^T(A_D\otimes X)\tilde{\vbeta}^*\right\|_\infty &=& 2\gamma \max_{l\in \{1,\ldots,r\}} \left\|\frac{1}{n}X^TX\sum_{k,l \in D_q, k\neq l}\frac{1}{|D_q|} (\beta^*_k-\beta^*_l)\right\|_\infty\\
&\leq& 2\gamma \rho_{\max} \max_{l\in \{1,\ldots,r\}} \left\|\sum_{k,l \in D_q}\frac{1}{|D_q|} (\beta^*_k-\beta^*_l)\right\|_2\\
&\leq& 2\gamma \rho_{\max}\acute{b}.
\end{eqnarray*}
Note that for $a \in D_q$ and $b \in D_q$ that 
\begin{eqnarray*}
\vU_a - \vU_b &=& \frac{1}{|D_q|} \left( \sum_{l \in D_q} \vbeta^*_l - \vbeta^*_a - \sum_{l \in D_q} \vbeta^*_l - \vbeta^*_b\right) \\
&=& \frac{1}{|D_q|} \sum_{l \in D_q} \vbeta^*_b - \vbeta^*_a = \vbeta^*_b - \vbeta^*_a.
\end{eqnarray*}
Therefore 
\begin{eqnarray*}
2\left\| \frac{2\gamma^2}{n(1+2\gamma)}(A_D\otimes X)^T(A_D\otimes X)\tilde{\vU}\right\|_\infty
&=&  \frac{4\gamma^2}{1+2\gamma} \max_{l\in \{1,\ldots,r\}} \left\|\frac{1}{n}X^TX\sum_{q=1}^Q I(l \in D_q) \sum_{k \in D_q} \frac{1}{|D_q|} (\vU_k-\vU_l)\right\|_\infty\\
&=& \frac{4\gamma^2}{1+2\gamma}\max_{l\in \{1,\ldots,r\}} \left\|\frac{1}{n}X^TX\sum_{q=1}^Q I(l \in D_q) \sum_{k \in D_q}\frac{1}{|D_q|} \vbeta^*_l-\vbeta^*_k \right\|_\infty\\
&\leq & \frac{4\gamma^2}{1+2\gamma} \rho_{\max} \max_{l\in \{1,\ldots,r\}} \left\|\sum_{q=1}^Q I(l \in D_q) \sum_{k \in D_q}\frac{1}{|D_q|} \vbeta^*_l-\vbeta^*_k\right\|_2\\
&\leq& \frac{4\gamma^2}{1+2\gamma}\rho_{\max}\acute{b}\\
&\leq& 2\gamma \rho_{\max}\acute{b}.
\end{eqnarray*}
Under assumptions A1 and A2 it follows that 
\begin{equation}
P\left(\left|\left|\frac{1}{n}\tilde{X}^T\tilde{\vE}\right|\right|_{\infty} > t\right) \leq 2\exp\left\{\frac{-nt^2}{2\sigma^2}+\log(rp)\right\}.
\end{equation}
Thus,
\begin{eqnarray*}
P\left\{\delta \geq 2 \left|\left|\nabla \ell(\tilde{\vbeta}')\right|\right|_{\infty} \right\}
 &\geq& P\left\{\delta \geq 2\left\|\frac{1}{n}\tilde{X}^T\tilde{\vE}\right\|_{\infty} + \rho_{\max}\acute{b}\left(\frac{4\gamma}{1+2\gamma} + 2\gamma + 2\gamma \right) \right\} \\
 &\geq& P\left(\delta \geq 2\left\|\frac{1}{n}\tilde{X}^T\tilde{\vE}\right\|_{\infty} + 8\gamma\rho_{\max}\acute{b}\right)\\
&\geq& P\left(\frac{3}{16}\delta \geq \left\|\frac{1}{n}\tilde{X}^T\tilde{\vE}\right\|_{\infty} \right) \\
&\geq& 1-2\exp\left\{\frac{-9n\delta^2}{16^22\sigma^2}+\log(rp)\right\} \\
&=& 1-2\exp\left\{-\frac{7}{2}\log(rp)\right\}.
\end{eqnarray*}
Set $c_1=2$ and $c_2=\frac{7}{2}$ and the proof is complete. 
\end{proof}

\subsection*{Proof of Theorem \ref{thm_strong}}
\begin{proof}
Applying the triangle inequality we have
\begin{equation}
\label{triangle}
\left\|\mbox{vec}\left(\hat{\beta}-\beta^*\right)\right\|_2 \leq \left\|\mbox{vec}\left(\hat{\beta}-\acute{\beta}\right)\right\|_2 + \left\|\tilde{\vbeta}'-\tilde{\vbeta}^*\right\|_2.
\end{equation}
For the second term using the upper bound for $\gamma$ stated in the conditions for Theorem \ref{thm_strong} and assumptions A4 and A5 it follows that
\begin{eqnarray*}
\left\|\tilde{\vbeta}'-\tilde{\vbeta}^*\right\|_2 &=& \frac{2\gamma}{1+2\gamma} \left\| \tilde{\vU}\right\|_2\\
	&\leq& 2\gamma \sqrt{s}\acute{b} \leq \frac{5\sigma}{2\rho_{\max}}\sqrt{\frac{s \log(rp)}{n}}.
\end{eqnarray*}
Combining the above inequality with \eqref{triangle} and Lemma \ref{acute_beta_lemma} it follows that there exists positive constants $c_1$ and $c_2$ such that 
\begin{equation*}
\left|\left|\mbox{vec}\left({\bar{B}}-{B^*}\right)\right|\right|_2 \leq \frac{48\sigma}{\kappa}\sqrt{\frac{s\log(rp)}{n}} + \frac{5\sigma}{2\rho_{\max}}\sqrt{\frac{s \log(rp)}{n}},
\end{equation*}
with probability at least $1-c_1\mbox{exp}(-c_2n\delta^2)$. To complete the proof set $c_3=48$ and $c_4=\frac{5}{2}$. 
\end{proof}

\bibliography{mcen_bib}

\end{document}